\documentclass{article}

\usepackage{microtype}
\usepackage{graphicx}
\usepackage{dsfont}
\usepackage{booktabs} % for professional tables

\usepackage[colorlinks=true,linkcolor=blue, citecolor=blue]{hyperref}
\usepackage{url}            % simple URL typesetting
\usepackage{amsfonts}       % blackboard math symbols
\usepackage{nicefrac}       % compact symbols for 1/2, etc.

\usepackage{placeins}
\usepackage{xcolor}
\usepackage{mathtools}
\usepackage{multirow}

\usepackage{graphicx}
\graphicspath{{./figures_/}}
\usepackage{subcaption}

\usepackage{amsthm,amsmath,amssymb}
\usepackage{array}
\usepackage{bm,bbm}
\usepackage{enumerate}
\usepackage{multirow}
\usepackage{makecell}
\theoremstyle{plain}
\newtheorem{theorem}{Theorem}
\newtheorem{definition}[theorem]{Definition}
\newtheorem{proposition}[theorem]{Proposition}
\newtheorem{corollary}[theorem]{Corollary}

\DeclareMathOperator*{\argmin}{arg\,min}
\DeclareMathOperator*{\argmax}{arg\,max}

\newcommand{\x}{\boldsymbol{x}}

\newcommand{\R}{\mathbb{R}}

\newcommand{\EJoint}{\mathop{\mathbb{E}}\limits_{(\x,y) \sim p(\x,y)}}
\newcommand{\Epos}{\mathop{\mathbb{E}}\limits_{\x \sim p(x|y=+1)}}
\newcommand{\Eneg}{\mathop{\mathbb{E}}\limits_{\x \sim p(x|y=-1)}}

\newcommand{\zerooneloss}{\ell_{01}}
\newcommand{\zeroonedloss}{\ell_{01c}}
\newcommand{\rejectionloss}{\ell_{01c}}

\newcommand{\rej}{\textregistered}

\newcommand{\vecone}{\mathds{1}}
\newcommand{\veceta}{\bm{\eta}}
\newcommand{\vecg}{\bm{g}}

\newcommand{\pyx}{p(y=+1|x)}

\newcommand\numberthis{\addtocounter{equation}{1}\tag{\theequation}}

\newcolumntype{L}{>{$}l<{$}}
\newcolumntype{C}{>{$}c<{$}}

\usepackage{hyperref}

\usepackage[accepted]{icml2021}

\icmltitlerunning{Classification with Rejection Based on Cost-sensitive Classification}

\begin{document}

\twocolumn[
\icmltitle{Classification with Rejection Based on Cost-sensitive Classification}

% It is OKAY to include author information, even for blind
% submissions: the style file will automatically remove it for you
% unless you've provided the [accepted] option to the icml2021
% package.

% List of affiliations: The first argument should be a (short)
% identifier you will use later to specify author affiliations
% Academic affiliations should list Department, University, City, Region, Country
% Industry affiliations should list Company, City, Region, Country

% You can specify symbols, otherwise they are numbered in order.
% Ideally, you should not use this facility. Affiliations will be numbered
% in order of appearance and this is the preferred way.
\icmlsetsymbol{equal}{*}

\begin{icmlauthorlist}
\icmlauthor{Nontawat Charoenphakdee}{to,ri}
\icmlauthor{Zhenghang Cui}{to,ri}
\icmlauthor{Yivan Zhang}{to,ri}
\icmlauthor{Masashi Sugiyama}{ri,to}
\end{icmlauthorlist}

\icmlaffiliation{to}{The University of Tokyo, Tokyo, Japan}
\icmlaffiliation{ri}{RIKEN AIP, Tokyo, Japan}

\icmlcorrespondingauthor{Nontawat Charoenphakdee}{nontawat@ms.k.u-tokyo.ac.jp}
% \icmlcorrespondingauthor{Eee Pppp}{ep@eden.co.uk}

% You may provide any keywords that you
% find helpful for describing your paper; these are used to populate
% the "keywords" metadata in the PDF but will not be shown in the document
\icmlkeywords{Machine Learning, ICML}

\vskip 0.3in
]

% this must go after the closing bracket ] following \twocolumn[ ...

% This command actually creates the footnote in the first column
% listing the affiliations and the copyright notice.
% The command takes one argument, which is text to display at the start of the footnote.
% The \icmlEqualContribution command is standard text for equal contribution.
% Remove it (just {}) if you do not need this facility.

\printAffiliationsAndNotice{}  % leave blank if no need to mention equal contribution
% \printAffiliationsAndNotice{\icmlEqualContribution} % otherwise use the standard text.

\begin{abstract}
The goal of classification with rejection is to avoid risky misclassification in error-critical applications such as medical diagnosis and product inspection.
In this paper, based on the relationship between classification with rejection and cost-sensitive classification, we propose a novel method of classification with rejection by learning an ensemble of cost-sensitive classifiers, which satisfies all the following properties:
(i) it can avoid estimating class-posterior probabilities, resulting in improved classification accuracy, 
(ii) it allows a flexible choice of losses including non-convex ones, 
(iii) it does not require complicated modifications when using different losses, 
(iv) it is applicable to both binary and multiclass cases, and 
(v) it is theoretically justifiable for any classification-calibrated loss.
Experimental results demonstrate the usefulness of our proposed approach in clean-labeled, noisy-labeled, and positive-unlabeled classification.
\end{abstract}

\section{Introduction}
In ordinary classification, a classifier learned from training data is expected to accurately predict a label of every possible test input in the input space.
However, when a particular test input is difficult to classify, forcing a classifier to always predict a label can lead to misclassification, causing serious troubles in risk-sensitive applications such as medical diagnosis, home robotics, and product inspection \citep{cortes2016boosting,geifman2017selective, ni2019calibration}.
To cope with this problem, classification with rejection was proposed as a learning framework to allow a classifier to abstain from making a prediction~\citep{chow1957, chow1970, bartlett2008classification, el2010foundations, geifman2017selective, cortes2016boosting,cortes2016learning, yuan2010,franc2019discriminative, pietraszek2005optimizing,gangrade2021selective}, so that we can prevent misclassification in critical applications.

A well-known framework for classification with rejection that has been studied extensively is called the cost-based framework~\citep{chow1970, bartlett2008classification,  yuan2010,cortes2016boosting,cortes2016learning, franc2019discriminative, ni2019calibration}.
In this setting, we set a pre-defined rejection cost to be less than the misclassification cost. As a result, a classifier trained in this framework prefers to reject than making a risky prediction, where there currently exist two main approaches as the following. 

The first approach is called the \emph{confidence-based approach}, where we train a classifier then use an output of the classifier as a confidence score~\citep{ bartlett2008classification,  grandvalet2009support,  herbei2006classification, yuan2010, ramaswamy2018consistent, ni2019calibration}. 
In this approach, we manually set a confidence threshold as a criterion to refrain from making a prediction, if the confidence score of a test input is lower than the threshold. 
Most confidence-based methods rely on a loss that can estimate class-posterior probabilities~\citep{yuan2010, reid2010,ni2019calibration}, which can be difficult to estimate especially when using deep neural networks~\citep{guo2017calibration}.
Although there are some exceptions that can avoid estimating class-posterior probabilities, most of them are only applicable to binary classification~\citep{bartlett2008classification,grandvalet2009support, manwani2015double}.  

The second approach is called the \emph{classifier-rejector} approach, where we simultaneously train a classifier and a rejector~\citep{cortes2016boosting,cortes2016learning, ni2019calibration}. 
It is known that this approach has theoretical justification in the binary case only for the exponential and hinge-based losses~\citep{cortes2016boosting, cortes2016learning}.
This is because the proof technique highly relies on the function form of the loss~\citep{cortes2016boosting,cortes2016learning}.
%but it is not simple to derive the calibration condition for a general loss because the proof highly relies on the function from of the loss~\citep{cortes2016learning, cortes2016boosting}.
In the multiclass case, ~\citet{ni2019calibration} argued that this approach is not suitable both theoretically and experimentally since the multiclass extension of~\citet{cortes2016learning} is not calibrated and the confidence-based softmax cross-entropy loss can outperform this approach in practice.

The goal of this paper is to develop an alternative approach to classification with rejection that achieves the following four design goals. 
First, it can avoid estimating class-posterior probabilities, since this often yields degradation of classification performance. 
Second, the choice of losses is flexible and does not require complicated modifications when using different losses, which allows a wider range of applications.
Third, it is applicable to both binary and multiclass cases. 
Fourth, it can be theoretically justified. 
In this paper, we show that this goal can be achieved by bridging the theory of cost-sensitive classification~\citep{elkan2001foundations,scott2012calibrated,steinwart2007compare} and classification with rejection.
The key observation that allows us to connect the two problems is based on the fact that one can mimic the Bayes optimal solution of classification rejection by only knowing $\argmax_y p(y|x)$ and whether $\max_y p(y|x) > 1-c$, where $c$ is the rejection cost.
Based on this observation, we propose the \emph{cost-sensitive approach}, which calibration can be guaranteed for \emph{any classification-calibrated loss}~\citep{zhang2004statistical,bartlett2006}.
Classification-calibration is known to be a minimum requirement for a loss in ordinary classification~\citep{bartlett2006}. 
This suggests that the loss choices of our proposed approach are \emph{as flexible as that of ordinary classification}. 

To emphasize the importance of having a flexible loss choice, we explore the usage of our approach for classification from positive and unlabeled data (PU-classification)~\citep{du2014, du2015convex, kiryo2017} and classification from noisy labels~\citep{oldccn, ghosh2015making}.
Our experimental results show that a family of symmetric losses, which are the losses that cannot estimate class-posterior probabilities~\citep{charoenphakdee2019symmetric}, can be advantageous in these settings.
%Symmetric losses are applicable to our cost-sensitive approach but cannot be used in the confidence-based approach without further modification if possible.
%And our cost-sensitive approach enables the use any classification-calibrated symmetric losses for classification with rejection.
We also provide experimental results of clean-labeled classification with rejection to illustrate the effectiveness of the cost-sensitive approach.

\section{Preliminaries}
In this section, we introduce the problem setting of classification with rejection. Then, we review cost-sensitive binary classification, which will be essential for deriving the proposed cost-sensitive approach for classification with rejection.

\subsection{Classification with Rejection}
Our problem setting follows the standard cost-based framework classification with rejection~\citep{chow1970,cortes2016learning,ni2019calibration}.
Let $\mathcal{X}$ be an input space and $\mathcal{Y}=\{1, \ldots, K\}$ be an output space, where $K$ denotes the number of classes. 
% Let $\x \in \mathcal{X}$ be a pattern and $y \in $ be a label, 
Note that we adopt a conventional notation $\mathcal{Y}= \{-1, +1 \}$ when considering binary classification~\citep{bartlett2006}.
In this problem, we are given the training input-output pairs $\{\x_i, y_i\}_{i=1}^{n}$ drawn i.i.d.~from an unknown probability distribution with density $p(\x, y)$. 
A classification rule of learning with rejection is $f\colon \mathcal{X} \to \{1, \ldots, K, \textrm{\rej} \}$, where $\text{\rej}$ denotes rejection. 
Let $c \in (0, 0.5)$ be the rejection cost. Unlike ordinary classification, where the zero-one loss $\zerooneloss(f(\x),y) = \vecone_{[f(\x) \neq y]}$\footnote{$\vecone_{[\cdot]}$ denotes an indicator function.} is the performance measure, we are interested in an extension of $\zerooneloss$, which is called the zero-one-$c$ loss $\rejectionloss$ defined as follows~\citep{ni2019calibration}:
%  \[ 
% \rejectionloss(f(\x), y) =
% \begin{cases}
% 1 & f(\x) \neq y \text{,}\\
% c & f(\x) =  \rej \text{,}\\
% 0 & f(\x) = y  \text{.}
% \end{cases}
% \]
 \[ 
\rejectionloss(f(\x), y) =
\begin{cases}
c & f(\x) =  \textrm{\rej} \text{,}\\
\zerooneloss(f(\x),y) & \text{otherwise.}
\end{cases}
\]
The goal is to find a classification rule $f$ that minimizes the expected risk with respect to $\rejectionloss$, i.e., 
\begin{align} \label{eq:0-1-c-risk}
    R^{\zeroonedloss}(f)=\EJoint[\rejectionloss(f(\x), y)].
\end{align}

In classification with rejection, a classification rule $f$ is allowed to refrain from making a prediction and will receive a fixed rejection loss $c$. 
In this paper, following most existing studies~\citep{cortes2016boosting,cortes2016learning,ramaswamy2018consistent,ni2019calibration}, we consider the case where $c<0.5$. 
Intuitively, this case implies that it is strictly better to reject if a classifier has less than half a chance to be correct. 
Thus, the case where $c<0.5$ is suitable if the goal is to avoid harmful misclassification.
We refer the readers to \citet{ramaswamy2018consistent} for more discussion on the case where $c\geq0.5$ and how it is fundamentally different from the case where $c<0.5$.
Next, let us define  $\veceta(\x) = [\eta_1(\x), \ldots, \eta_K(\x)]^\top$, where $\eta_y(x)=p(y|\x)$ denotes the class-posterior probability of a class $y$.
The optimal solution for classification with rejection $f^*=\argmin_f\,R^{\zeroonedloss}(f)$ known as Chow's rule~\citep{chow1970} can be expressed as follows:
% \[
% g^*(\x)=
% \begin{cases}
% \mathrm{1} & \eta(\x)>1-c \text{,} \\ 
% \textrm{\rej} & c \leq \eta(\x) \leq 1-c \text{,}\\
% \mathrm{-1} & \eta(\x) < c \text{,}
% \end{cases}
% \]
\begin{definition}[Chow's rule~\citep{chow1970}]\label{def:chow-mul} 
\rm{The optimal solution of multiclass classification with rejection $f^*=\argmin_f\,R^{\zeroonedloss}(f)$ can be expressed as
\begin{equation*}
f^*(\x)=
\begin{cases}
\textrm{\rej} & \max_y \eta_y(\x) \leq 1-c \text{,}\\
\argmax_y \eta_y(\x) & \text{otherwise.}
\end{cases}
\end{equation*}
}
\end{definition}
Chow's rule suggests that classification with rejection is solved if we have the knowledge of $\veceta(\x)$. 
Therefore, one approach is to estimate $\veceta(\x)$ from training examples. 
This method is in a family of the confidence-based approach, which has been extensively studied in both the binary~\citep{yuan2010} and multiclass cases~\citep{ni2019calibration}. 
Figure~\ref{fig:confbased} illustrates the confidence-based approach.

\begin{figure*}
\centering
\includegraphics[width=0.8\textwidth]{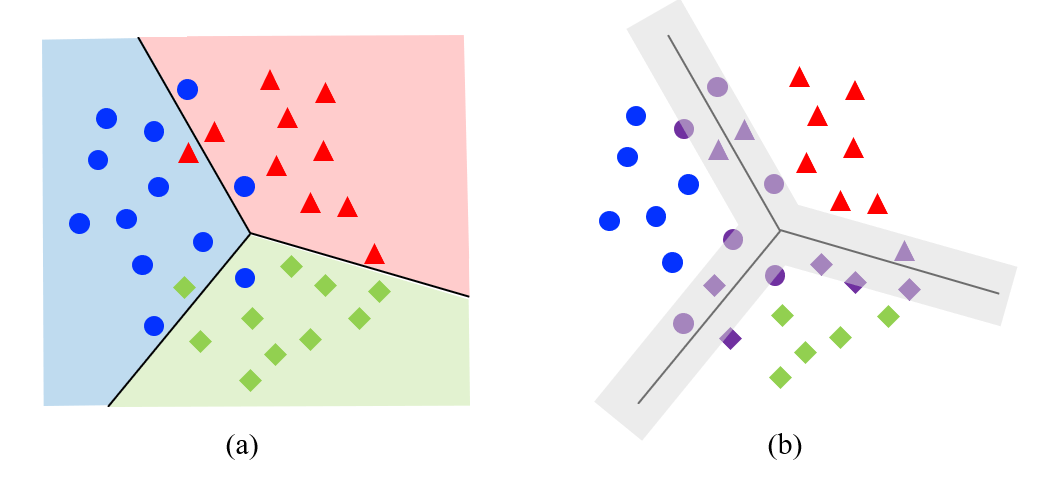}
\caption{\label{fig:confbased} 
Illustration of the confidence-based approach.
Figure (a) denotes a prediction function. 
The rejector in figure (b) has a rejection region spreads from the decision boundary of the prediction function. The width of the rejection region depends on the choice of the rejection threshold parameter.
Data points in purple are rejected.
}
\vspace{-0.1cm}
\end{figure*}
% one can argue that all class probability estimates are not required to solve classification with rejection, but only the threshold points of Chow's rule, which was studied by~\citet{grandvalet2009support}. Furthermore, a recent approach where the rejection function and prediction function are separated and trained simultaneously has also been explored~\citep{cortes2016boosting, cortes2016learning}. 

\subsection{Cost-sensitive Binary Classification}
Consider binary classification where $y \in \{-1, +1 \}$. 
In ordinary classification, the false positive and false negative costs are treated equally. 
%Therefore, it suffices to use the zero-one loss $\zerooneloss$ on each test data point to evaluate the performance.
On the other hand, in cost-sensitive classification, the false positive and false negative costs are generally unequal~\citep{elkan2001foundations,saerens2002adjusting, scott2012calibrated}.
%An illustrative example is the problem of  life-threatening disease detection. 
%It could be safer to identify a healthy person as being ill (false positive) than to identify that a person with disease is healthy (false negative).

Without loss of generality, we define $\alpha \in (0,1) $ to be the false positive cost and $1-\alpha$ to be the false negative cost. 
Then, the expected cost-sensitive risk can be expressed as
\begin{align*}
    R^{\zerooneloss}_\alpha(f)&=(1-\alpha)\pi\Epos[\zerooneloss(f(\x), +1)] \\ 
    &\quad +\alpha(1-\pi)\Eneg[\zerooneloss(f(\x), -1)],
\end{align*}
where $\pi=p(y=+1)$ denotes the class prior.

It is known that the Bayes optimal cost-sensitive binary classifier can be expressed as follows:

\begin{definition}[\citet{scott2012calibrated}]\label{def:scott-optimal} 
\rm{The optimal cost-sensitive classifier $f_{\alpha}^*=\argmin_f\R_\alpha(f)$ can be expressed as
\begin{equation*}
   f_{\alpha}^*(\x)=
\begin{cases}
\mathrm{+1} & p(y=+1|\x)>\alpha \text{,} \\
\mathrm{-1} & \text{otherwise.}
\end{cases}
\end{equation*}
} 
\end{definition}
Note that when $\alpha=0.5$, the Bayes optimal solution $f_{0.5}^*(\x)$ coincides with that of ordinary binary classification. 
Moreover, when $\alpha$ is known, cost-sensitive binary classification is solved if we have access to the class-posterior probability $p(y=+1|\x)$.

\section{Cost-sensitive Approach}
In this section, we propose a cost-sensitive approach for classification with rejection. We begin by describing our motivation and analyzing the behavior of the Bayes optimal solution of classification with rejection.
Then, we show that this problem can be solved by simultaneously solving multiple cost-sensitive classification problems.
\subsection{Motivation}
As suggested by Chow's rule~\citep{chow1970},  classification with rejection can be solved by estimating the class-posterior probabilities.  
However, an important question arises as: \emph{Is class-posterior probability estimation indispensable for solving classification with rejection?}  
This question is fundamentally motivated by Vapnik's principle~\citep{vapnik1998statistical}, which suggests not to solve a more general problem as an intermediate step when solving a target problem if we are given a restricted amount of information.
% “When solving a problem of interest, do not solve a more general problem as an intermediate step”.
% \textit{"If you possess a restricted amount of information
% for solving some problem, try to solve the
% problem directly and never solve a more general
% problem as an intermediate step. It is possible
% that the available information is sufficient for a
% direct solution but is insufficient for solving a more
% general intermediate problem."}
% \textit{``If you possess a restricted amount of information
% for solving some problem, try to solve the
% problem directly and never solve a more general
% problem as an intermediate step."}

% In our context, the general problem is class-posterior probability estimation. 
% In fact, knowing class-posterior probabilities can also solve many other problems, e.g., divergence estimation~\citep{qin1998inferences,bickel2007discriminative}, density ratio estimation~\citep{sugiyama2012density}, and linear-fractional metric optimization~\citep{dembczynski2013optimizing, koyejo2014consistent}. 
In our context, the general problem is class-posterior probability estimation. 
In fact, knowing class-posterior probabilities can also solve many other problems~\citep{qin1998inferences,bickel2007discriminative,sugiyama2012density,dembczynski2013optimizing, koyejo2014consistent}. 
However, many of such problems are also known to be solvable without estimating the class-posterior probabilities~\citep{ulsif,bao2019calibrated}. 
Note that class-posterior probability estimation can be unreliable when the model is misspecified~\citep{begg1990consequences,heagerty2001misspecified} %and also when the model is 
or highly flexible~\citep{guo2017calibration, hein2019relu}. 

To find a more direct solution for classification with rejection, we seek for a general approach that it may not be able to estimate class-posterior probabilities, but its optimal solution coincides with the optimal Chow's rule~\citep{chow1970}.
Although the idea of directly solving classification with rejection without class-posterior estimation itself is not novel,
most existing methods are only applicable to binary classification with rejection~\citep{bartlett2008classification, grandvalet2009support, manwani2015double, cortes2016learning, cortes2016boosting}, or focus on specific types of losses~\citep{ramaswamy2018consistent}.
For the multiclass case,~\citet{zhang2018reject} proposed to modify a loss by bending it to be steeper (for the hinge loss) and positively unbounded, but there exist hyperparameters to be tuned such as the rejection threshold and the bending slope.
Also,~\citet{mozannar2020consistent} recently proposed a classifier-rejector approach by augmenting a rejection class in the model's prediction, but their loss choice is limited to the modified cross-entropy loss.  
More discussion on related work is provided in Appendix~\ref{app:related_work}.

\begin{figure}
\centering
\includegraphics[width=\columnwidth]{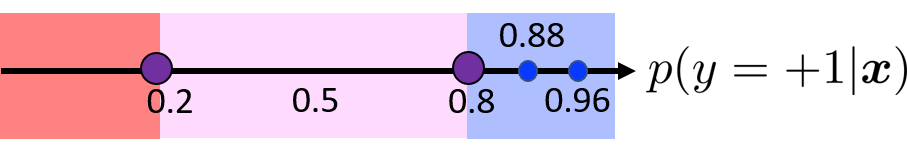}
\caption{\label{fig:chow-intuition} 
Illustration of Chow's rule in binary classification with rejection and the unnecessity of knowing the class-posterior probability to solve this problem. 
If the rejection cost $c=0.2$, as long as we know $p(y=1|x) > 0.8$, knowing the exact value of the class-posterior probability does not change our final decision to predict a positive label. 
}
\vspace{-2mm}
\end{figure}

\subsection{A Closer Look at Chow's Rule}\label{sec:sub-closer-look}

Here, we analyze the behavior of Chow's rule~\citep{chow1970}. 
We discuss the minimal knowledge %that is
required for a classification rule to mimic Chow's rule, which illustrates %that class-posterior probabilities are not required to mimic Chow's rule.
that the class-posterior probabilities need not to be known.
%Also, we use the insight of this fact to develop our proposed approach in Section~\ref{sec:sub-proposed}.
For simplicity, we begin by considering binary classification with rejection.

In binary classification with rejection, Chow's rule in Definition~\ref{def:chow-mul} can be expressed as
\begin{equation}\label{eq:chow-bin}
f^*(\x)=
\begin{cases}
\mathrm{1} & \pyx>1-c \text{,} \\
\textrm{\rej} & c \leq \pyx \leq 1-c \text{,}\\
\mathrm{-1} & \pyx < c \text{.}
\end{cases}
\end{equation}
To solve binary classification with rejection, there are only three conditions to verify, which are $\pyx > 1-c$, $\pyx <c$, and $\pyx \in [c,1-c]$. 
We can see that if we know $\pyx > 1-c$, we do not need to know the exact value of $\pyx$ to predict the label as positive. 
For example, if $c=0.2$, knowing $p(y=+1|x) > 0.8$ is already sufficient to predict a label, i.e., knowing whether $p(y=+1|x) > 0.88$ or $p(y=+1|x) > 0.96$ does not change the decision of Chow's rule.
Figure~\ref{fig:chow-intuition} illustrates this fact, which is the key intuition why it is possible to develop a method that can avoid estimating the class-posterior probabilities for solving this problem.

\subsection{Binary Classification with Rejection Based on Cost-sensitive Classification}\label{sec:sub-proposed-binary}
Here, we show that by solving two cost-sensitive binary classification problems, binary classification with rejection can be solved.
The following proposition shows the relationship between the optimal solutions of cost-sensitive binary classification and that of binary classification with rejection.
\begin{proposition}
\label{prop:chow-scott} 
In binary classification with rejection, Chow's rule can be expressed as
\begin{equation}
f^*(\x)=
\begin{cases}
\mathrm{1} & f_{1-c}^*(\x) =1 \text{,} \\
\mathrm{-1} & f_{c}^*(\x) = -1 \text{,} \\
\textrm{\rej} & \mathrm{otherwise}. \\ 
\end{cases}
\end{equation}
\end{proposition}
\begin{proof}
We assert that if we can verify the following two conditions:
\begin{align}
    \pyx > 1-c \label{ineq:binary1},\\
    \pyx > c  \label{ineq:binary2},
\end{align}
then binary classification with rejection is solved.
Based on Chow's rule~\eqref{eq:chow-bin},
if Ineq.~\eqref{ineq:binary1} holds,  Ineq.~\eqref{ineq:binary2} must also hold since $c<0.5$. 
Then we should predict a positive label. 
On the other hand, we should predict a negative label if Ineq.~\eqref{ineq:binary2} does not hold.
Next, if  Ineq.~\eqref{ineq:binary2} holds but Ineq.~\eqref{ineq:binary1} does not hold, we should reject $\x$.
As a result, based on Definition~\ref{def:scott-optimal}, knowing $f_{1-c}^*(\x)$ and $f_{c}^*(\x)$ is sufficient to verify Ineqs.~\eqref{ineq:binary1} and Ineq.~\eqref{ineq:binary2}. 
This concludes the proof.
\end{proof}
Proposition~\ref{prop:chow-scott} suggests that by solving two binary cost-sensitive classification with $\alpha=c$ and $\alpha=1-c$ to obtain $f_{c}^*(\x)$ and $f_{1-c}^*(\x)$, binary classification with rejection can be solved.

% From this observation, we can solve binary classification with rejection by verifying Ineqs.~\eqref{ineq:binary1} and~\eqref{ineq:binary2}. 
% This is where we can borrow the literature of binary cost-sensitive classification problem.

% Empirical risk minimization

% Figure of losses
\begin{table*}
\centering
\caption{Classification-calibrated binary surrogate losses and their properties including the convexity, symmetricity (i.e., $\phi(z)+\phi(-z)$ is a constant), and its capability to estimate the class posterior probability $\eta_1(\x)$ in binary classification.
The column ``confidence-based'' indicates that a loss is applicable to the confidence-based approach and it satisfies all conditions required in order to use the previous work to derive its excess risk bound~\citep{yuan2010,ni2019calibration}. 
On the other hand, our cost-sensitive approach can guarantee the existence of the excess risk bound as long as a loss $\phi$ is classification-calibrated~\citep{bartlett2006}.
%, and the derivation of the excess risk bound based on our Theorem~\ref{thm:excess}
%can easily borrow the results of cost-sensitive classification~\citep{steinwart2007compare,scott2012calibrated}. 
\label{table:cc-binary-loss}}

\begin{tabular}{|C|C|C|C|C|C|C|C|}
\hline
\text{Loss name} & \phi(z) &  \text{Convex} & \text{Symmetric} & \text{Estimating $\eta_1(\x)$}& \text{Confidence-based}\\ \hline
% \text{Zero-one} & -0.5 \mathrm{sign}(z) + 0.5&  \times &\checkmark&  \times\\ 
\text{Squared} & (1-z)^{2} &  \checkmark &\times&\checkmark&\checkmark\\ 
\text{Squared hinge} & \max(0, 1-z)^{2} &  \checkmark &\times&\checkmark&\checkmark \\ 
\text{Exponential} & \exp(-z)&  \checkmark &\times&\checkmark &\checkmark\\ 
\text{Logistic} & \mathrm{log}(1+\exp(-z)) &  \checkmark &\times &\checkmark&\checkmark\\ 
\text{Hinge} & \max(0, 1-z) &  \checkmark &\times &  \times&\times\\ 
\text{Savage}  & \left[(1+\exp(2z))^{2}\right]^{-1} &\times &  \times &\checkmark&\times\\ 
\text{Tangent} & (2\mathrm{arctan}(z)-1)^{2}  &\times &  \times & \checkmark&\times\\ 
\text{Ramp}& \mathrm{max}(0, \mathrm{min}(1, 0.5-0.5z))& \times &  \checkmark & \times &\times \\ 
 \text{Sigmoid} & \left[1+\exp(z)\right]^{-1} & \times &  \checkmark &  \times&\times\\ 
%   \text{Unhinged} & 1-z  &\checkmark &  \checkmark &  \times&\times\\
     \hline
\end{tabular}
\end{table*}
\subsection{Multiclass Extension}\label{sec:sub-proposed-multi}
Here, we show that our result in Section~\ref{sec:sub-proposed-binary} can be naturally extended to the multiclass case.
More specifically, we show that multiclass classification with rejection can be solved by learning an ensemble of $K$ binary cost-sensitive classifiers.

Let us define the Bayes optimal solution for one-versus-rest cost-sensitive binary classifier $f^{*,y}_{\alpha}$, where $y$ is the positive class and $y' \in \mathcal{Y}, y' \neq y$ are the negative classes:
\begin{equation*}
   f^{*,y}_{\alpha}(\x)=
\begin{cases}
\mathrm{+1} & \eta_y(\x) >\alpha \text{,} \\
\mathrm{-1} & \text{otherwise.}
\end{cases}
\end{equation*}

Then, we obtain the following proposition (its proof can be found in Appendix~\ref{proof:prop-chow-mult}).
\begin{proposition}
\label{prop:chow-scott-multi} 
Chow's rule in multiclass classification with rejection can be expressed as
\begin{equation*}
f^*(\x)=
\begin{cases}
\textrm{\rej} & \max_y f^{*,y}_{1-c}(\x) = -1 \text{,}\\
\argmax_y f^{*,y}_{1-c}(\x) & \rm{otherwise.}
\end{cases}
\end{equation*}
\end{proposition}
Proposition~\ref{prop:chow-scott-multi} suggests that by learning cost-sensitive classifiers $f^{*,y}_{1-c}$ for $y \in \mathcal{Y}$, it is possible to obtain Chow's rule without estimating the class-posterior probabilities.
Note that when $c<0.5$, there exists at most one $y'\in\mathcal{Y}$ such that $f^{*,y'}_{1-c}(\x) = 1$. 
This is because it implies that $\eta_{y'}(x) > 1-c$, which is larger than $0.5$. 

\noindent \textbf{Related work:} The similar idea of constructing cost-sensitive classifiers for solving classification with rejection has also been recently explored in the bounded-improvement framework by~\citet{gangrade2021selective}. 
Unlike the cost-based framework considered in this paper, where our goal is to minimize the expected risk with respect to $\rejectionloss$ in Eq.~\eqref{eq:0-1-c-risk}, the goal of the bounded-improvement framework is to learn a classifier that minimizes the number of rejections (i.e., maximizing coverage) while achieving at least the pre-defined accuracy on non-rejected data~\citep{pietraszek2005optimizing,el2010foundations, geifman2017selective,liu2019deep,franc2019discriminative}.
~\citet{gangrade2021selective} proposed to solve classification with rejection under the bounded-improvement framework by solving multiple one-sided classification problems. 
Then, they further relaxed one-sided classification problems to cost-sensitive classification problems.
Since the optimal solution in the bounded-improvement framework is not necessarily Chow's rule, their final learning objective is different from ours.
Please see~\citet{gangrade2021selective} for more details. 

% Thus, it is impossible to have $y\neq y'$ such that $ > 0.5$ more than one $y'$ would make the sum of the probability vector $\veceta(\x)$ to be more than one, which is impossible.
% We want to emphasize the limitation that the relationship requires the rejection cost to be less than $0.5$. 
% In the binary case, it is strictly better to not reject if $c>0.5$~\citep{chow1970,yuan2010} and thus one can simply use ordinary classification for such $c$.
% In the multiclass case, the case where $c>0.5$ implies that it is preferable to predict even when $\max_y \eta_y(\x) < 0.5$.  
% From the application perspective, we argue that this confidence score is too low to be useful for critical applications that require classification with rejection in general, and thus the case where $c<0.5$ would be more relevant in the real-world applications.
% Nevertheless, it is interesting from the theoretical perspective to explore the case where $c>0.5$ and investigate the usefulness of potential applications.
% We refer the readers to the work by~\citet{ramaswamy2018consistent} for more discussion on the difficulties of the case where $c>0.5$.

\section{A Surrogate Loss for the Cost-sensitive Approach}\label{sec:surrogate}
%In the previous section, we propose to solve classification with rejection by making use of the relationship between classification with rejection and cost-sensitive classification.
In this section, we propose a surrogate loss for the cost-sensitive approach for classification with rejection.
%by making use of the relationship between classification with rejection and cost-sensitive classification, 

It is known that given training data, directly minimizing the empirical risk with respect to $\zeroonedloss$ is computationally infeasible~\citep{bartlett2008classification,ramaswamy2018consistent}.
Therefore, many surrogate losses have been proposed to learn a classifier with rejection in practice~\citep{yuan2010,cortes2016learning,ni2019calibration}. Here, we propose the cost-sensitive surrogate loss for classification with rejection.
Let $\vecg(\x) = [g_1(\x), \ldots, g_K(\x)]^\top$, where $g_y(x)\colon \mathcal{X} \to \R$ is the score function for a class $y$. 
Let $\phi\colon\R\to\R$ be a binary margin surrogate loss.
A margin loss is a %well-known 
class of loss functions for binary classification that takes only one real-valued argument~\citep{bartlett2006,reid2010}.
Table~\ref{table:cc-binary-loss} illustrates examples of binary margin surrogate losses.
With a choice of %binary margin surrogate loss 
$\phi$, we can define our proposed cost-sensitive surrogate loss as follows.
\begin{definition}
\label{def:proposed-surr}
\rm{Given a binary margin surrogate loss~$\phi$ and a pre-defined rejection cost $c$, \emph{the cost-sensitive surrogate loss} for classification with rejection is defined as}
\begin{align*}
    % \label{eq:proposed-surr}
    \mathcal{L}^{c,\phi}_\mathrm{CS}(\vecg; \x, y) = c \phi \big( g_{y}(\x)\big) + (1-c) \sum_{y' \neq y} \phi \big( -g_{y'}(\x) \big).
\end{align*}
\end{definition}

\begin{figure*}
\centering
\includegraphics[width=0.8\textwidth]{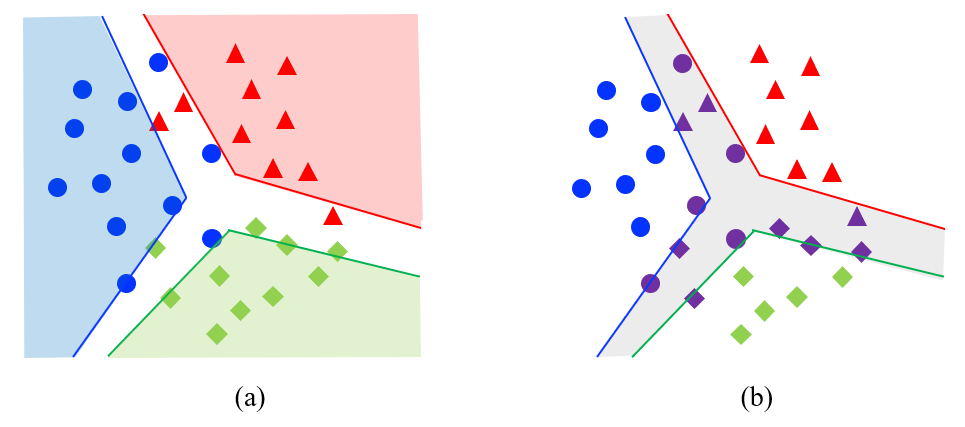}
\caption{\label{fig:cs-approach} 
Illustration of the cost-sensitive approach.
Figure (a) denotes a prediction function.
Unlike the confidence-based approach~(Figure~\ref{fig:confbased}), the prediction function is not designed to predict all data points in the space and the rejection region does not spread from the decision boundary. 
The decision boundary is based on an ensemble of cost-sensitive classifiers for blue, red, and green classes, respectively.
Then, the rejector in figure (b) is constructed based on the rejection rule in Cond.~\eqref{eq:amb-reject} by aggregating the prediction result of each cost-sensitive classifier. 
Data points in purple are rejected.
}
\vspace{-1.5mm}
\end{figure*}

Next, following the empirical risk minimization framework~\citep{vapnik1998statistical}, a learning objective function can be straightforwardly obtained as follows:
\begin{equation}
\label{eq:erm}
    \hat{R}^{\mathcal{L}^{c,\phi}_\mathrm{CS}}(\vecg) = \frac{1}{n}\sum_{i=1}^{n} \mathcal{L}^{c,\phi}_\mathrm{CS}(\vecg; \x_i, y_i). 
\end{equation}
Note that regularization can also be applied in practice to avoid overfitting. 
Moreover, we want to emphasize that although it is theoretically suggested to learn an ensemble of classifiers to solve classification with rejection, in practice, by using linear-in-parameter models or neural networks with $K$-dimensional vectorial outputs, we can conveniently learn all $K$ cost-sensitive binary classifiers together at once, which is $\vecg$. 

%After learning $\vecg$ by minimizing Eq.~\eqref{eq:erm} using training data, we have to design a rejection rule for it.
After learning $\vecg$ by minimizing Eq.~\eqref{eq:erm}, we have to design how to reject $\x$.
%Next, we design a rejection rule for $\vecg$ trained by by minimizing Eq.~\eqref{eq:erm}.
Following the optimal rejection rule in our Proposition~\ref{prop:chow-scott-multi}, i.e., $\max_y f^{*,y}_{1-c}(\x) = -1$%. i.e., we reject $\x$ if all cost-sensitive binary classifiers predict negatively
, we can directly obtain the following rejection rule:
\begin{equation}
\label{eq:amb-reject}
    \max_y g_y(\x) \leq 0 \text{.}
\end{equation}
Intuitively, Cond.~\eqref{eq:amb-reject} suggests to reject $\x$ if all $g_y(\x)$ have low prediction confidence. %to predict class $\y$.
One may interpret this type of rejection as \emph{distance rejection}~\citep{dubuisson1993statistical}, where the rejection is made when $\vecg$ is uncertain whether $\x$ belongs to any of the known classes.
Note that this does not necessarily imply that $\x$ belongs to an unknown class, e.g., $\x$ may be located close to the decision boundary, causing none of $g_y(\x)$ to be confident enough to predict a class $y$.

Next, we also consider the following rejection rule:
\begin{align}
\label{eq:morethanone-reject}
        \exists y,y' \, \text{s.t.} \, y\neq y'  \text{and} \, g_y(\x), g_{y'}(\x) > 0.
\end{align}
% Cond.~\eqref{eq:morethanone-reject} suggests to reject if \emph{at least two} binary classifiers $g_y(\x), g_{y'}(\x)$, which leads to prediction conflict.
Cond.~\eqref{eq:morethanone-reject} suggests to reject $\x$ because there exists a prediction conflict among at least two binary classifiers, i.e., $ g_y(\x)$ suggests to predict a class $y$ but $g_{y'}(\x)$ suggests to predict another class $y'$.
Note that if we succeed to obtain the optimal classifier $\vecg^*$, this condition is \emph{impossible} to be satisfied.
Recall that in Section~\ref{sec:sub-proposed-multi}, for $\vecg^*$, at most one $g^*_y(\x)$ can be more than zero since it implies $\eta_y > 1-c > 0.5$. 
Nevertheless, Cond.~\eqref{eq:morethanone-reject} may hold in practice due to empirical estimation.
This rejection condition can be interpreted as \emph{ambiguity rejection}~\citep{dubuisson1993statistical}, where the rejection is made when $\vecg$ interprets $\x$ to be associated with more than one class.
More discussion on the proposed rejection conditions is provided in Appendix.~\ref{sec:add-exp}.
%\footnote{We found that adding this condition can slightly improve the performance in our experiment. 
%More discussion on the proposed rejection conditions is provided in Appendix.~\ref{sec:add-exp}}

To sum up, we employ the following classification rule for the cost-sensitive approach:
\begin{align}
\label{eq:final-rule}
f(\x; \vecg)=
\begin{cases}
\textrm{\rej} & \text{Conds.}~\eqref{eq:amb-reject} \, \text{or} \, \eqref{eq:morethanone-reject} \text{,}\\ %\max_y g_y(\x) \leq 0 \text{,}\\
% \textrm{\rej} & %\makecell[l]{\exists y,y' \, \text{s.t.} \, y\neq y' \\ \text{and} \, g_y(\x), g_{y'}(\x) > 0\text{,}}\\
% \text{Cond.}~\eqref{eq:morethanone-reject} \text{,}\\
\argmax_y g_y(\x) & \rm{otherwise.}
\end{cases}
\end{align}

Figure~\ref{fig:cs-approach} illustrates the cost-sensitive approach.
It is worth mentioning that our rejection condition is different from that of~\citet{zhang2018reject}. 
In their rejection rule, an input $\x$ is rejected if all binary classifiers' outputs are close to zero.
In our case, Cond.~\eqref{eq:amb-reject} rejects~$\x$ as long as all $g_y(\x)$'s are negative, e.g., $\x$ is also rejected if all prediction outputs are much smaller than zero. 
Also, their method can predict a set of labels when at least two classifiers predict positively, which is different from our problem setting, where it is only allowed to predict one label or refrain from making a prediction. 
%Moreover, we do not have any hyperparameter for the rejection threshold.  
% This classification rule $f(\x ; \vecg)$ is intuitive and can be interpreted as follows. 
% It will refrain from making a prediction if all binary classifiers $g_y(\x)$ predict the negative class, which will happen when none of the one-versus-rest classifiers have enough confidence to regard $\x$ as positive w.r.t.~its class of interests.
% Moreover, it will also reject if there exist at least two binary classifiers that predict positively, which can be treated as a conflict among binary classifiers.
% If both conditions for rejection are not satisfied, then $f(\x ; \vecg)$ will give a prediction of the class with the highest score, namely $\argmax_y g_y(\x)$.
\section{Theoretical Analysis}
In this section, we show that the classification rule $f(\x;\vecg)$ in Eq.~\eqref{eq:final-rule} can achieve Chow's rule and also provide excess risk bounds. 
%Our analysis is based on proving the calibration results and excess risk bounds.
%, which emphasizes that any classification-calibrated binary surrogate losses can be used, and excess risk bounds can be derived based on the relationship between classification with rejection and cost-sensitive classification. 

% We discuss a condition for making our proposed formulation calibrated to the optimal solution of classification with rejection.
% Furthermore, an excess risk bound for the one-versus-all cost-sensitive approach is provided, which shows that our results are general and is applicable to any classification-calibrated loss.
% Our results are all proven in a multiclass case. 

\subsection{Calibration}
We begin by introducing the well-known notion of classification-calibrated loss in binary classification.
%, which is quite well-studied~\citep{zhang2004statistical,bartlett2006,scott2012calibrated,steinwart2007compare}. 
Let us define the pointwise conditional surrogate risk for a fixed input $\x$ in binary classification with its class-posterior probability of a positive class $\eta_1$: 
\begin{align}
    C^{\phi}_{\eta_1}(v) = \eta_1\phi(v) + (1-\eta_1)\phi(-v),
\end{align}
for $v\in\R$. Next, a classification-calibrated loss can be defined as follows.

\begin{definition}[\citet{bartlett2006}]
We say a loss $\phi$ is classification-calibrated if for any $\eta_1 \neq \frac{1}{2}$, we have
\begin{align*}
    \inf_{v(2\eta_1-1) \leq 0} C^{\phi}_{\eta_1}(v) > \inf_{v} C^{\phi}_{\eta_1}(v).
\end{align*}
\end{definition}
Intuitively, classification-calibration ensures that minimizing a loss $\phi$ can give the Bayes optimal binary classifier $\text{sign}(2\eta_1-1)$.
%Classification-calibration ensures that $v$ which has an inconsistent sign with the Bayes optimal binary classifier $\text{sign}(2\eta_1-1)$ must have a strictly higher conditional risk $C^{\phi}_{\eta_1}(v)$. 
%As a result, minimizing a loss forces $v$ to have the same sign with $\text{sign}(2\eta_1-1)$, which encourages the trained classifier to have the correct prediction.
It is known that a convex loss $\phi$ is classification-calibrated if and only if it is differentiable at $0$ and $\phi'(0)<0$~\citep{bartlett2006}.
% Although it is trickier for non-convex losses, it also has been investigated extensively ~\citep{reid2010,charoenphakdee2019symmetric}.

Analogously, in classification with rejection, calibration is also an important property that has been used to verify if a surrogate loss is appropriate~\citep{bartlett2008classification,yuan2010,cortes2016learning,cortes2016boosting,ni2019calibration}.
Calibration guarantees that by replacing the zero-one-$c$ loss $\rejectionloss$ with a surrogate loss $\mathcal{L}$, the optimal Chow's rule can still be obtained by minimizing the surrogate risk. 
% solution with respect to a new loss will also give the reasonable solution for the the original risk with respect to the zero-one-$c$ loss $\rejectionloss$ in Eq.~\eqref{eq:0-1-c-risk}.
Verifying the calibration condition in classification with rejection has not been as well-studied as ordinary binary classification.
We are only aware of the works by~\citet{yuan2010} and ~\citet{ni2019calibration}, which provided a condition to verify calibration of general loss functions for the confidence-based approach.
Nevertheless, their condition can only verify a convex loss. 
Note that losses that are calibrated in our cost-sensitive approach may not be calibrated in the confidence-based approach, e.g., the sigmoid loss.
%it requires a loss to be convex and also more restricted conditions are restricted for a loss to use their condition to verify the calibration result.
See Table~\ref{table:cc-binary-loss} for more details.

Now we are ready to show that the calibration condition of our proposed approach is equivalent to the classification-calibration condition of $\phi$.
Let us define the pointwise conditional surrogate risk $W_\mathcal{L}$ of an input $\x$ with its class-posterior probability $\veceta(\x)$ for the multiclass case:
\begin{align}
    W_\mathcal{L} \big(\vecg(\x)); \veceta(\x) \big) = \sum_{y \in \mathcal{Y}} \eta_y(\x) \mathcal{L}\big( \vecg; \x, y \big).
\end{align}
By analyzing the classification rule with respect to the conditional risk minimizer, we obtain the following theorem (its proof can be found in Appendix~\ref{proof:calib}).
\begin{theorem}
\label{thm:calib}
Let $g^*$ be a conditional risk minimizer that minimizes the pointwise conditional surrogate risk $g^*(\x) = \argmin_g W_{\mathcal{L}^{c,{\phi}}_\mathrm{CS}} \big(\vecg(\x); \veceta(\x)))$.
The surrogate loss $\mathcal{L}^{c,{\phi}}_\mathrm{CS}$ is calibrated for classification with rejection, that is, $f(\x;\vecg^*)=f^*(\x)$ for all $\x \in \mathcal{X}$, if and only if $\phi$ is classification-calibrated.
\end{theorem}
%The proof is based on combining the results of ~\citet{bartlett2006}, \citet{scott2012calibrated} and our Proposition~\ref{prop:chow-scott-multi}.
Theorem~\ref{thm:calib} suggests that the condition to verify if our cost-sensitive surrogate loss $\mathcal{L}^{c,{\phi}}_\mathrm{CS}$ is calibrated is equivalent to the condition of whether $\phi$ is classification-calibrated.
As long as a binary surrogate loss $\phi$ is classification-calibrated, minimizing the surrogate risk w.r.t.~$\mathcal{L}^{c,{\phi}}_\mathrm{CS}$ can lead to the optimal Chow's rule.
As a result, Theorem~\ref{thm:calib} successfully borrows the knowledge in the literature of binary classification to help prove calibration in classification with rejection for the cost-sensitive approach.
%This suggests that the problem of verifying the calibration condition of our cost-sensitive approach is as simple as verifying a well-studied problem of verifying whether a binary surrogate loss $\phi$ is classification-calibrated.

%% discuss the previous attempts: convexity conditions, difficult to obtain the result EVEN in the binary case

% Here, we show that with the one-versus-all formulation, we only need a binary loss function to be \emph{classification-calibrated} in order to make our formulation calibrated with respect to classification with rejection.

%% Theorem: f(x;g^*) = chow's rule if \phi is classification-calibrated

%% explain more about classification-calibrated loss (zhang2004, bartlett2006)

\begin{figure*}[t]
\centering
\includegraphics[width=\linewidth]{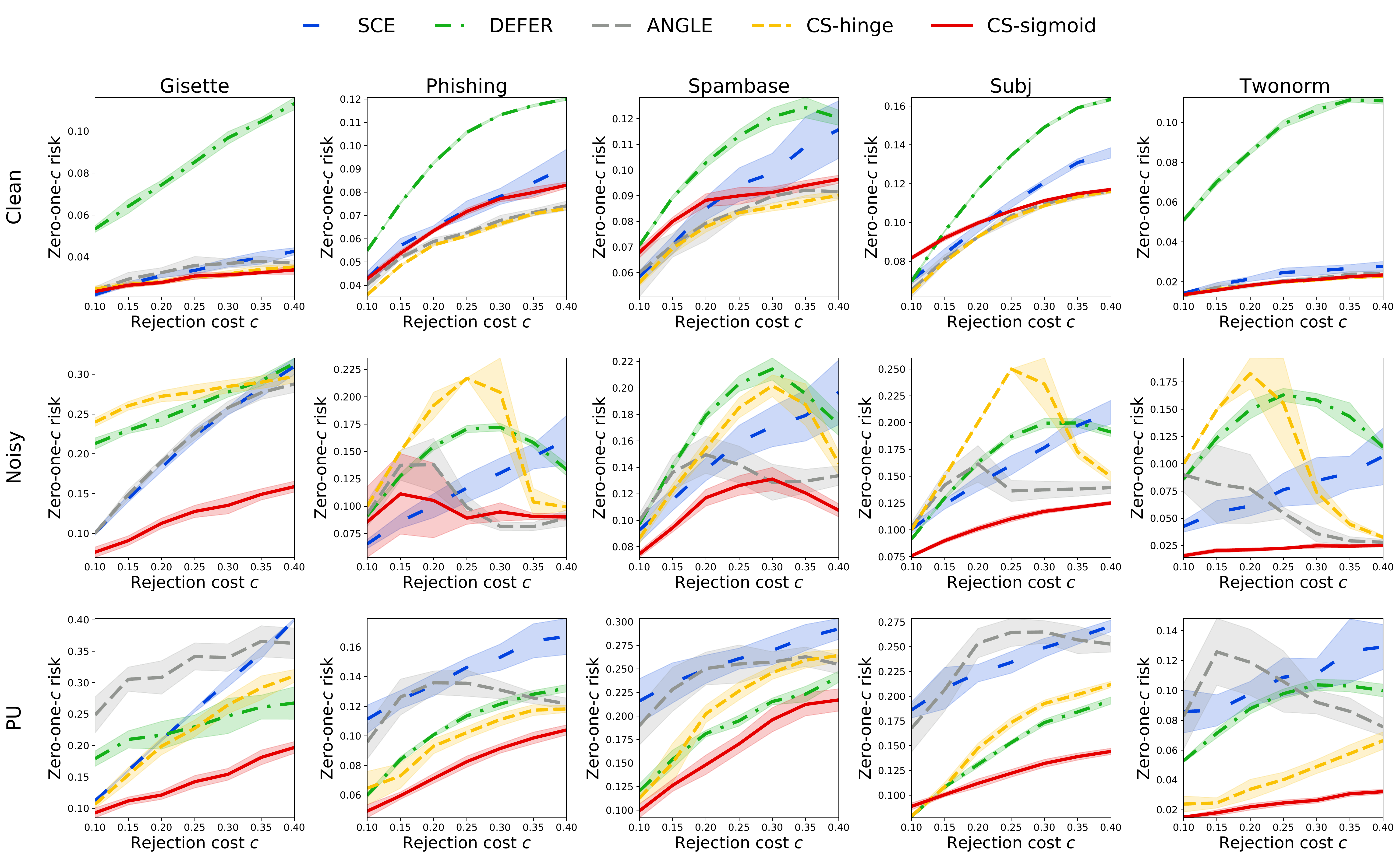}
% \vspace{-5mm}
\caption{Mean and standard error of the test empirical zero-one-$c$ risk over ten trials with varying rejection costs (Binary classification).
Each column indicates the performance with respect to one dataset.
(Top) clean-labeled classification with rejection. 
(Middle) noisy-labeled classification with rejection. 
(Bottom) PU-classification with rejection.}
\label{fig:bin-01c}
% \vspace{-3mm}
\end{figure*}

\subsection{Excess Risk Bound}
While calibration ensures that the optimal solution w.r.t.~a surrogate loss agrees with the optimal Chow's rule, an excess risk bound provides a regret bound relationship between the zero-one-$c$ loss $\zeroonedloss$ and a surrogate loss $\mathcal{L}$.

Let $R^{\phi,i}_{1-c}(g_i)$ be the cost-sensitive binary surrogate risk for class $i$ and $R^{\mathcal{L},*}$ be the minimum risk w.r.t.~to the loss~$\mathcal{L}$.
We prove the following theorem, which is our main result to use for deriving the excess risk bound of the cost-sensitive approach for any classification-calibrated loss (its proof can be found in Appendix~\ref{proof:excess}).

\begin{theorem}
\label{thm:excess}
Consider a cost-sensitive surrogate loss $\mathcal{L}^{c,{\phi}}_\mathrm{CS}$. 
Let $f$ be a classification rule of the cost-sensitive approach with respect to the score function $\vecg$, that is, $f=f(\x;\vecg)$ for an input $\x$.
If a binary surrogate loss $\phi$ is classification-calibrated, the excess risk bound can be expressed as follows:
\begin{align}
    R^{\zeroonedloss}(f)-R^
    {\zeroonedloss,*} &\leq  R^{ \mathcal{L}^{c,{\zerooneloss}}_\mathrm{CS}}(\vecg)-R^{ \mathcal{L}^{c,{\zerooneloss}}_\mathrm{CS},*} 
    \label{eq:excess-01}
    \\
    &\leq \sum_{i=1}^K \psi_{\phi,1-c}^{-1}(R^{\phi,i}_{1-c}(g_i)-R_{1-c}^{\phi,i,*}),
    \label{eq:excess-binary}
\end{align}
where $\psi_{\phi,1-c}\colon \R\to\R$ is a non-decreasing invertible function and $\psi_{\phi,1-c}(0)=0$.
\end{theorem}
Ineq.~\eqref{eq:excess-01} suggests that the regret of the classification with rejection problem can be bounded by the regret of the cost-sensitive surrogate with respect to the zero-one loss~$\zerooneloss$. 
This inequality allows us to borrow the existing findings of cost-sensitive classification to  give excess risk bounds for classification with rejection.
Next, Ineq.~\eqref{eq:excess-binary} suggests that $R^{ \mathcal{L}^{c,{\zerooneloss}}_\mathrm{CS}}(\vecg)-R^{ \mathcal{L}^{c,{\zerooneloss}}_\mathrm{CS},*}$ is bounded by the sum of an invertible function of the regret of the cost-sensitive binary classification risk.
The invertible function $\psi_{\phi,1-c}$ is a well-studied function in the literature of cost-sensitive classification, which is guaranteed to exist for any classification-calibrated loss~\citep{steinwart2007compare,scott2012calibrated}.
For example, $\psi^{-1}_{\phi,1-c}(\epsilon) = \frac{\epsilon^2}{2c(1-c) - (\epsilon)(1-2c)}$ for the squared loss, where $\epsilon \geq 0$.
Examples of $\psi_{\phi,1-c}$ for more losses and how to derive $\psi_{\phi,1-c}$ can be found in~\citet{steinwart2007compare} and~\citet{scott2012calibrated}.
Since $\psi_{\phi,1-c}$ is non-decreasing and $\psi_{\phi,1-c}(0)=0$, the regret with respect to the zero-one-$c$ loss will also get smaller and eventually become zero if the surrogate risk is successfully minimized.

As an example to demonstrate how to obtain an excess risk bound with our Theorem~\ref{thm:excess}, we prove that the following excess risk bound holds for the hinge loss $\phi_{\mathrm{hin}}$, which is the loss that cannot estimate the class-posterior probabilities~\citep{svm}, and its optimal solution for the confidence-based approach cannot mimic Chow's rule. The bound can be straightforwardly derived based on our Theorem~\ref{thm:excess} and the known fact that $\psi^{-1}_{\phi_{\mathrm{hin}},1-c}(\epsilon) =~\epsilon$~\citep{steinwart2007compare}.
\begin{corollary}
Let us consider the hinge loss $\phi_{\mathrm{hin}}(z)= \max(0, 1-z)$. The excess risk bound for the cost-sensitive surrogate $\mathcal{L}^{c,{\phi_{\mathrm{hin}}}}_\mathrm{CS}$ can be given as follows: 
\begin{align*}
    R^{\zeroonedloss}(f)-R^{\zeroonedloss,*} 
  %  &\leq \sum_{i=1}^K \left( R^{\phi_{\mathrm{hin}},i}_{1-c}(g_i)-R_{1-c}^{\phi_{\mathrm{hin}},i,*}\right) \\
    &\leq R^{ \mathcal{L}^{c,{\phi_{\mathrm{hin}}}}_\mathrm{CS}}(\vecg)-R^{ \mathcal{L}^{c,{\phi_{\mathrm{hin}}}}_\mathrm{CS},*} .
\end{align*}
\end{corollary}
%% shows table based on the result of steinwart

% \begin{table*}[t]
% \centering
% \caption{Mean and standard error of the average test empirical zero-one-$c$ risk over ten trials of the clean-labeled setting (rescaled to $0-100$). 
% The average score was calculated w.r.t.~all rejection costs.
% Outperforming methods are highlighted in boldface using one-sided t-test with the significance level $5\%$.
% }
% \begin{tabular}{lcccccc}
% \toprule
% Dataset
% & SCE & DEFER & OVA-squared& CS-squared& CS-hinge & CS-sigmoid \\
% \midrule
% Gisette
% & $3.14(0.47)$ & $8.63(2.07)$ & $9.60(1.91)$ 
% & $8.96(1.82)$ & $\mathbf{2.96(0.40)}$ & $\mathbf{2.89(0.36)}$ \\
% Phishing
% & $7.10(1.69)$ & $9.73(2.26)$ & $7.68(0.90)$ 
% & $6.65(1.16)$ & $\mathbf{5.86(1.23)}$ & $6.64(1.34)$ \\
% Spambase
% & $9.55(2.29)$ & $10.48(1.81)$ & $11.25(1.64)$ 
% & $9.41(1.57)$ & $\mathbf{7.77(1.18)}$ & $8.65(0.95)$ \\
% Subj
% & $10.76(2.34)$ & $12.70(3.23)$ & $11.17(1.82)$ 
% & $\mathbf{9.99(1.53)}$ & $\mathbf{9.63(1.77)}$ & $10.24(1.20)$ \\
% Twonorm
% & $2.50(0.66)$ & $9.17(2.12)$ & $5.06(1.07)$ 
% & $\mathbf{1.94(0.15)}$ & $\mathbf{1.90(0.32)}$ & $\mathbf{1.89(0.31)}$ \\
% MNIST
% & $0.84(0.09)$ & $4.40(0.96)$ & $1.18(0.35)$ 
% & $0.80(0.16)$ & $\mathbf{0.66(0.19)}$ & $1.41(0.92)$ \\
% Fashion-MNIST
% & $8.45(0.18)$ & $9.71(2.24)$ & $\mathbf{6.55(1.07)}$ 
% & $6.99(1.59)$ & $\mathbf{6.81(1.34)}$ & $8.23(1.57)$ \\
% KMNIST
% & $5.76(0.22)$ & $9.65(2.59)$ & $5.19(0.42)$ 
% & $\mathbf{4.51(1.05)}$ & $\mathbf{4.26(1.07)}$ & $5.72(1.50)$ \\
% \bottomrule
% \end{tabular}
% \label{tab:all-01c}
% \vspace{-3mm}
% \end{table*}

\begin{figure*}
\centering
\includegraphics[width=\linewidth]{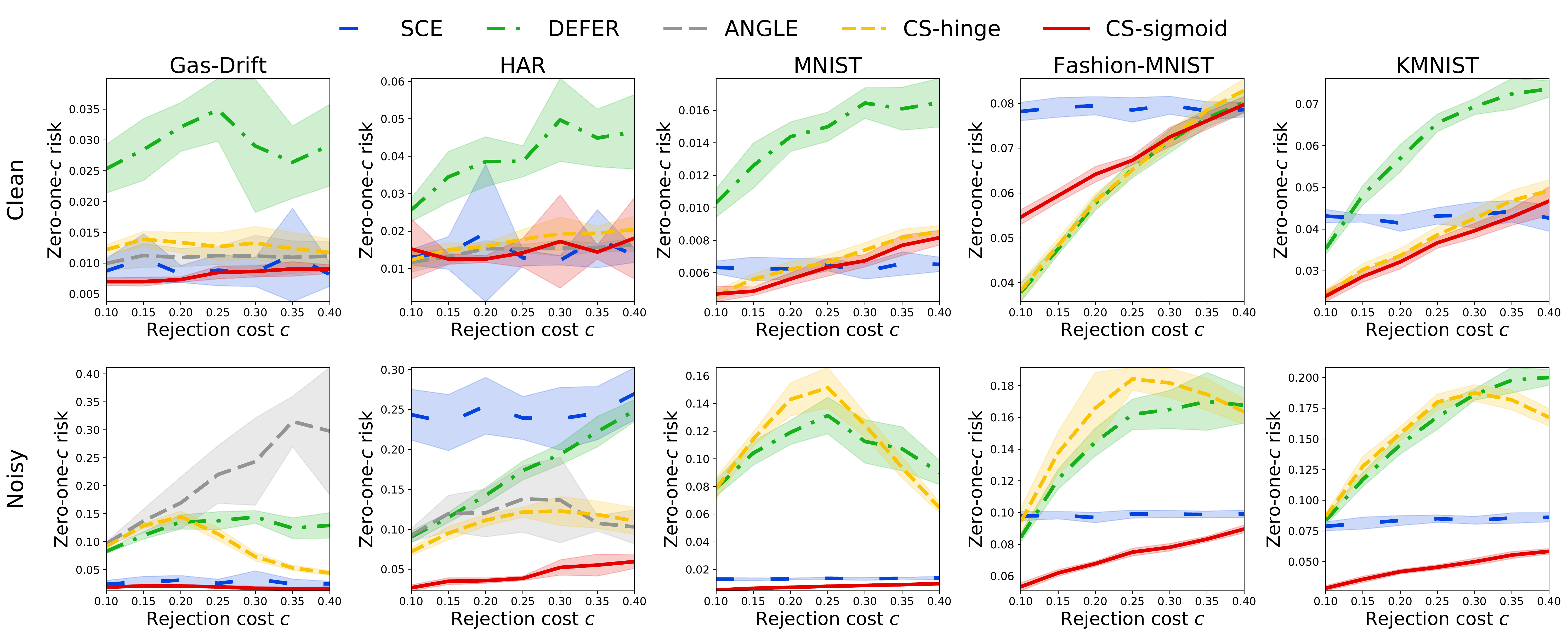}
\vspace{-5mm}
\caption{%
Mean and standard error of the test empirical zero-one-$c$ risk over ten trials with varying rejection cost (Multiclass classification). 
Each column indicates the performance with respect to one dataset.
(Top) clean-labeled classification with rejection. 
(Bottom) noisy-labeled classification with rejection.
For MNIST, Fashion-MNIST, and KMNIST, we found that ANGLE failed miserably and has  zero-one-$c$ risk more than $0.5$ and thus it is excluded from the figure for readability.
}
\vspace{-2mm}
\label{fig:mul}
\end{figure*}

\section{Experimental Results}
In this section, we provide experimental results of classification with rejection.
The evaluation metric is the test empirical zero-one-$c$ risk over ten trials.
We also reported the rejection ratio, accuracy of accepted data, and the full experimental results in the table format in Appendix~\ref{sec:add-exp}. 
The varying rejection costs ranged from $\{0.1,0.15,0.20,0.25,0.30,0.35,0.40\}$ for all settings.
For noisy-labeled classification, we used the uniform noise~\citep{oldccn, ghosh2015making}, where the randomly selected $25\%$ of the training labels were flipped.

\subsection{Experiment Setup}
\textbf{Datasets and models:} 
For binary classification, we used the subjective-versus-objective classification (Subj), which is a text dataset~\citep{pang2004sentimental}.
Moreover, we used Phishing and Spambase, which are tabular datasets, and Twonorm, which is a synthetic dataset drawn from different multivariate Gaussian distributions~\citep{dataset3}. 
We also used the Gisette dataset, which is the problem of separating the highly confusible digits $4$ and $9$ with noisy features~\citep{guyon2005result}. 
Linear-in-input models were used for all binary datasets.
For multiclass classification, we used Gas-Drift~\citep{vergara2012chemical} and Human activity recognition (HAR)~\citep{anguita2013public}, which are tabular datasets. 
Multilayer perceptrons were used for both datasets.
We also used the image datasets, which are MNIST~\citep{lecun1998mnist}, Kuzushiji-MNIST (KMNIST)~\citep{clanuwat2018deep}, and Fashion-MNIST~\citep{xiao2017fashion}.
Convolutional neural networks were used for all image datasets.
The implementation was done using PyTorch~\citep{pytorch2019}. 
More details on the datasets and implementation can be found in Appendix~\ref{sec:exp-detail}.

\textbf{Methods:} 
For the confidence-based approach, based on~\citet{ni2019calibration}, we used the softmax cross-entropy loss (SCE). %one-versus-rest (OVA-squared) with the squared loss for the confidence-based methods.
For the classifier-reject approach, we used the proposed method by~\citet{mozannar2020consistent} (DEFER).
We also used the method by~\citet{zhang2018reject} with the bent hinge loss (ANGLE).
%, which one may regard as a classifier-reject approach, although the rejection rule partially depends on the classifier's confidence. 
%For our proposed cost-sensitive approach, we used the squared (CS-squared), hinge (CS-hinge), and sigmoid (CS-sigmoid) losses. 
For our cost-sensitive approach, we used the hinge (CS-hinge), and sigmoid (CS-sigmoid) losses. 

\textbf{Hyperparameter tuning:}
We provided additional training data for SCE and ANGLE to tune their hyperparameters.
For SCE, we also added temperature scaling~\citep{guo2017calibration} to improve the prediction confidence. 
For ANGLE, we chose the bending slope paramater according to~\citet{zhang2018reject} and tuned the rejection threshold.
In PU-classification, it is difficult to tune hyperparameters for them. 
Thus, we provided clean-labeled data for them \emph{only} for hyperparameter tuning.
Both rejection threshold of ANGLE and the temperature parameter for SCE are chosen from the following candidate set of twenty numbers spaced evenly in a log scale from $0$ to $1$ (inclusively) and nine integers from $2$ to $10$.
%: \{$0.001,$ $0.00143845,$ $0.00206914,$ $0.00297635,$ $0.00428133,$ $0.00615848,$ $0.00885867,$ $0.01274275,$ $0.01832981,$ $0.02636651,$ $0.0379269,$ $0.05455595,$ $0.078476,$ $0.11288379,$ $0.16237767,$ $0.23357215,$ $0.33598183,$ $0.48329302,$ $0.6951928,$ $1,$ $2,$ $3,$ $4,$ $5,$ $6,$ $7,$ $8,$ $9,$ $10$\}, which is the set%
Since only SCE and ANGLE require additional data to tune hyperparameters, it is not straightforward to provide a fair comparison because our methods and DEFER do not use validation data.
Nevertheless, with less data, our methods are still competitive and can outperform the baselines in several settings. 

% \begin{figure*}[t]
% \centering
% \includegraphics[width=\linewidth]{mnist.pdf}
% \vspace{-5mm}
% \caption{Mean and standard error of the test empirical zero-one-$c$ risk over ten trials of clean and noisy settings for multiclass classification datasets: MNIST, Fashion-MNIST, and KMNIST. 
% }
% \label{fig:mult-01c}
% \vspace{-3mm}
% \end{figure*}

\subsection{Binary Classification with Rejection}
Here, we compare the performance of all methods in clean-labeled, noisy-labeled, and positive-unlabeled classification with rejection.
For PU-classification, we implemented all methods based on the empirical risk minimization framework proposed by~\citet{kiryo2017} (more detail can be found in Appendix~\ref{sec:exp-detail}).
%We provide more information on how to implement PU-classification method using a surrogate loss in Appendix~\ref{sec:exp-detail}.

%Table~\ref{tab:all-01c} shows the average of the performance over all rejection costs for both binary and multiclass data in the clean-labeled setting.
Figure~\ref{fig:bin-01c} shows the performance with varying rejection costs for all settings. 
In clean-labeled classification, it can be observed that CS-hinge and ANGLE are most preferable in this setting.
%It is also interesting to see that CS-squared outperformed OVA-squared for most datasets.
%This can be interpreted by the fact that OVA-squared requires the estimation of the class-posterior probability to be exactly correct while CS-squared only requires to achieve the sign of the prediction correctly.
%As long as the sign of the prediction is correct, the incorrect estimation of the class-posterior probability does not affect the final classification rule in the cost-sensitive approach.
%This demonstrates the benefit of avoiding class-posterior probability estimation.
%Note that SCE performed very well compared to other baselines, which coincides to the result of~\citet{ni2019calibration}.
In noisy-labeled and PU classification, CS-sigmoid outperformed other methods in most cases.
%On the other hand, CS-hinge did not perform well in the noisy-labeled setting although it is preferable in the clean-labeled setting.
This illustrates the usefulness of having a flexible choice of loss functions. 
%Note that symmetric losses such as the sigmoid loss is not allowed in the confidence-based approach.
We also found that noise can degrade the performance to be worse than always reject for some methods.
Moreover, we found that DEFER rejected data more often than other methods, which may sometimes lead to worse performance.
In PU-classification, SCE and ANGLE did not perform well although clean labeled data were used for hyperparameter tuning. 
This could be due to a steep loss can suffer severely from the negative risk problem~\citep{kiryo2017}, causing them to be ineffective in PU-classification. 
% In PU-classification, where it is not straightforward to tune hyperparameters, we found that SCE and ANGLE are less preferable.
% This illustrates the usefulness of our cost-sensitive approach over the baselines that require additional hyperparameters.

% Hyper-parameters
% We used linear-in-input models for all experiments in this section.
% We used the Adam optimizer, the learning rate was set to $0.0002$ and the weight decay rate was set to $0.0001$.
% We repeated $10$ trials for each setting and ran $200$ epochs for each single trial.

\subsection{Multiclass Classification with Rejection}
Figure~\ref{fig:mul} illustrates the performance of all methods in the clean-labeled and noisy-labeled settings.
It can be observed that SCE had almost the same performance for all rejection costs. 
Although temperature scaling is applied, it seems that SCE still suffered from overconfidence~\citep{guo2017calibration} and failed to reject the ambiguous data points.
This could be due to SCE has the high accuracy on the validation set (more than $90\%$) and thus temperature scaling could not smoothen the prediction confidence to reject the ambiguous data effectively.
%We speculate that this is because the accuracy of SCE is considerably high and thus the  
%because it suffered from 
Interestingly, DEFER did not suffer from such overconfidence although it is also based on the cross-entropy loss and it rejected the data more than other methods.
%OVA-squared performed better than CS-squared in Fashion-MNIST and KMNIST when $c \in \{0.35,0.4\}$.
For ANGLE, we found that although it can perform competitively in Gas-drift and HAR, it failed miserably in the image datasets. 
For figure's readability, we report the performance of ANGLE in a table format in Appendix~\ref{sec:add-exp}.
%We included the result of ANGLE in Appendix.
In noisy-labeled classification, CS-sigmoid outperformed other methods in most cases.

\section{Conclusions}
We have proposed a cost-sensitive approach to classification with rejection, where any classification-calibrated loss can be applied with theoretical guarantee. 
Our theory of excess risk bounds explicitly connects the classification with rejection problem to the cost-sensitive classification problem.
Our experimental results using clean-labeled, noisy-labeled, and positive and unlabeled training data demonstrated the advantages of avoiding class-posterior probability estimation and having a flexible choice of loss functions. 
%For future work, it is interesting to explore the open-set classification setting where data that belong to unknown classes can be observed in the test phase. 

\section*{Acknowledgements}
We would like to thank Han Bao, Takeshi Teshima, Chenri Ni, and Junya Honda for helpful discussion, and also the Supercomputing Division, Information Technology Center, The University of Tokyo, for providing us the Reedbush supercomputer system to conduct the experiments.
NC was supported by MEXT scholarship, JST AIP Challenge, and Google PhD Fellowship program. 
ZC was supported by JST AIP Challenge.
MS was supported by the International Research Center for Neurointelligence (WPI-IRCN) at The University of Tokyo Institutes for Advanced Study.

% \bibliography{refs}
% \bibliographystyle{icml2021}

\appendix
\onecolumn
\section{Related work}
\label{app:related_work}
In this section, we provide more discussion on the relationship of our work with~\citet{zhang2018reject} and~\citet{mozannar2020consistent}.

\subsection{Our work and~\citet{zhang2018reject}} ~\citet{zhang2018reject} considers to tackle classification with rejection using angle-based classification approach.

Given $\x$ and $y$. Define
\begin{equation*}
\mathbf{y}_y= 
\begin{cases}
(K-1)^{-\frac{1}{2}} \mathbf{1}_{K-1}& y=1 \text{,}\\
-(1+K^{\frac{1}{2}})/ \{(K-1)^{\frac{3}{2}}\}  \mathbf{1}_{K-1} + \{ K/ (K-1)\}^\frac{1}{2} \mathbf{e}_{y-1} & 2 \leq y \leq K,
\end{cases}
\end{equation*}
where $\mathbf{e}_{y-1}$ denotes the one-hot vector with one at the index $y-1$ and $\mathbf{1}_{K-1} \in R^{k-1}$ denotes a vector of all ones.
Next, let $a \in \R$ be a positive scalar. 
~\citet{zhang2018reject} proposed the following bent hinge loss:
\begin{align*}
    \mathcal{L}^{\mathrm{hin}}_\mathrm{ANGLE}(\vecg; \x, y) = \sum_{y' \neq y} \phi_{\mathrm{hin}} \big(-\mathbf{y}_y^\top \vecg(\x) \big),
\end{align*}
where
\begin{align*}
\phi_{\mathrm{hin}}(u)=
\begin{cases}
1-au & u<0\\ 
1-u & 0 \leq u \leq 1 \text{,}\\
0 & \rm{otherwise.}
\end{cases}
\end{align*}
and the bent distance weighted discrimination loss:
\begin{align*}
    \mathcal{L}^{\mathrm{dwd}}_\mathrm{ANGLE}(\vecg; \x, y) = \sum_{y' \neq y} \phi_{\mathrm{hin}} \big(-\mathbf{y}_y^\top \vecg(\x) \big),
\end{align*}
where
\begin{align*}
\phi_{\mathrm{dwd}}(u)=
\begin{cases}
1-au & u<0\\ 
1-u & 0 \leq u \leq 0.5 \text{,}\\
\frac{1}{4u} & \rm{otherwise.}
\end{cases}
\end{align*}

Let $S_{\delta}(v) = \mathrm{sign}(v) \max (|v|-\delta ,0)$. 
After training a classifier with their proposed loss function, the decision rule can be expressed as 
\begin{align*}
f(\x; \vecg)=
\begin{cases}
\textrm{\rej} &  \forall j \, \, \text{s.t.} \, \, S_{\delta}(\mathbf{y}_j^\top \vecg(\x)) = 0, \\ 
\argmax_y \mathbf{y}_y^\top \vecg(\x) & \rm{otherwise.}
\end{cases}
\end{align*}
In their rejection rule, an input $\x$ is rejected if all binary classifiers' outputs are close to zero.
In our cost-sensitive approach, Cond.~\eqref{eq:amb-reject} rejects~$\x$ as long as all $g_y(\x)$'s are negative, e.g., $\x$ is also rejected if the all prediction outputs are much smaller than zero. 
Moreover, we do not have any hyperparameter in our rejection rule.

% \begin{align*}
%     S_{\delta}(\mathbf{y}_j^\top \vecg(\x)) = 0, \forall j,
% \end{align*}
% where $S_{\delta}(v) = \mathrm{sign}(v) \max (|v|-\delta ,0)$.
% If the rejection condition does not hold, the prediction output is $\argmax_j \mathbf{y}_j^\top \vecg(\x)$.

There are two hyperparameters that are needed to be tuned for the angle-based method, which are a bending slope $a$ and a rejection threshold $\delta$\footnote{There is also a hyperparameter for determining the regularization strength, which is omitted here for brevity.}.
\citet{zhang2018reject} defined the following quantities given the rejection cost $c$ and the number of clases $K$:
\begin{align*}
    a_1&=\frac{K-1-c}{Kc-c}, \quad a_2=\frac{(K-1)(1-c)}{c}.
    % a_2&=\frac{(K-1)(1-c)}{c}. \\
\end{align*}
It was suggested that the choice of of $a$ can be chosen by either $a=a_1$ or $a=a_2$.  
For $\delta$, it needs to be tuned with the validation set with respect to the validation empirical zero-one-$c$ risk. 

It can be observed that our approach and~\citet{zhang2018reject} are different. 
The modification of~\citet{zhang2018reject} is based on angle-based classification while it is based on the cost-sensitive one-vs-rest loss in our case.
Moreover, no additional hyperparameter is introduced and we do not modify the loss by bending it to be steeper (for the hinge loss). 
Note that the proposed bending scheme of~\citet{zhang2018reject} will lead any loss to be positively unbounded, including the symmetric losses, which may cause them to lose their favorable properties. 
Also, it is not straightforward to design a bending scheme for any loss to the best of our knowledge. 
In our approach, any classification-calibrated loss can be straightforwardly applied in the surrogate loss in Definition~\ref{def:proposed-surr}.
Finally, our Theorem~\ref{thm:excess} made it possible to transfer the excess risk bound from cost-sensitive classification~\citep{steinwart2007compare,scott2012calibrated} to classification with rejection.
We are not aware of other works that discuss a theory that can explicitly show the excess risk bound relationship between cost-sensitive classification and cost-sensitive classification.

It is worth noting that~\citet{zhang2018reject} also considered a different setting called classification from refine options, where a classifier is also allowed to predict a set of labels instead of one label. 
\begin{align}
\label{eq:zhang-rule}
f(\x; \vecg)=
\begin{cases}
\textrm{\rej} &  \forall j \, \, \text{s.t.} \, \, S_{\delta}(\mathbf{y}_j^\top \vecg(\x)) = 0, \\ 
\{j:  S_{\delta}(\mathbf{y}_j^\top \vecg(\x)) > 0\} &  \exists j \, \, \text{s.t.} \, \, S_{\delta}(\mathbf{y}_j^\top \vecg(\x)) > 0, \\ 
\{j:  S_{\delta}(\mathbf{y}_j^\top \vecg(\x)) = 0\} & \rm{otherwise.}
\end{cases}
\end{align}
It can be observed that the second condition in Eq.~\eqref{eq:zhang-rule}, i.e., $\exists j \, \, \text{s.t.} \, \, S_{\delta}(\mathbf{y}_j^\top \vecg(\x)) > 0$ is similar to our Cond.~\eqref{eq:morethanone-reject}, which occurs when a classifier wants to predict more than one classes.
Nevertheless, in our problem, it is not allowed to predict a set and we propose to reject a data point in this scenario.
It is also interesting to explore the problem of learning with refine options with our proposed cost-sensitive approach and we leave it for future work.

\subsection{Our work and~\citet{mozannar2020consistent}}
Recently,~\citet{mozannar2020consistent} has proposed a method for classification with rejection based on a reduction to cost-sensitive learning. 
However, the reduction scheme is different to ours
since they proposed to augment a rejection class in the model and the loss choice is fixed to the cross-entropy loss.
The main idea is to augment a rejection class $K+1$ in the score function $\vecg$.
Rejection will be made if the maximum score of $\vecg$ is at index $K+1$.
Given the rejection cost $c$ and a score function $\vecg:\R^d \to \R^{K+1}$, the loss function proposed by~\citet{mozannar2020consistent} can be expressed as 
\begin{align*}
    % \label{eq:proposed-surr}
    \mathcal{L}^{c}_\mathrm{DEFER}(\vecg; \x, y) = \mathrm{log}\left(\frac{\exp(g_y(\x))}{\sum_{j=1}^{K+1} \exp(g_j(\x))}\right)  + (1-c) \log \left(\frac{\exp(g_{K+1}(\x))}{\sum_{j=1}^{K+1} \exp(g_j(\x))}\right).
\end{align*}

It can be seen that our cost-sensitive approach gives a different form of loss function, that is, their loss function is not a special case of our cost-sensitive approach and vice versa.
Moreover, the theoretical analysis in~\citet{mozannar2020consistent} is based on analyzing this specific loss. 
Unlike our work, it may not be straightforward to borrow the theory of cost-sensitive classification~\citep{steinwart2007compare,scott2012calibrated} to justify the theoretical properties of a general surrogate loss function for classification with rejection in their approach.
It is worth pointing out that one advantage of the loss function proposed by~\citet{mozannar2020consistent} over our approach is that it is applicable to the situation where the rejection cost can be different for each $\x$ (see~\citet{mozannar2020consistent} for more detail).
Nevertheless, in our problem setting, the rejection cost is assumed to be a constant.

\section{Proofs}
In this section, we provide proofs for the theoretical results in the main body.
\subsection{Proof of Proposition~\ref{prop:chow-scott-multi}}
Based on the following Chow's rule:
\begin{equation*}
f^*(\x)=
\begin{cases}
\textrm{\rej} & \max_y \eta_y(\x) \leq 1-c \text{,}\\
\argmax_y \eta_y(\x) & \text{otherwise.}
\end{cases}
\end{equation*}
It is straightforward to see that we can mimic Chow's rule by only knowing whether
\begin{align}
    \label{app:cond1}
    \eta_y(\x) \leq 1-c \quad \text{for all} \quad y \in \mathcal{Y},
\end{align}
and 
\begin{align*}
    \argmax_y \eta_y(\x). 
\end{align*}

This is because if Condition~\eqref{app:cond1} is true, then a classifier refrains from making a prediction, otherwise a classifier predicts a class $\argmax_y \eta_y(\x)$, which matches Chow's rule.  

To verify whether $\eta_y(\x) \leq 1-c$ for a class $y$, it suffices to learn a cost-sensitive binary classifier where $\alpha=1-c$ to classify between a target class $y$ and other classes (i.e., one-versus-rest classifier), as suggested in Definition~\ref{def:scott-optimal}.
We define such optimal cost-sensitive binary classifier as $f^{*,y}_{1-c}$.
As a result, we can construct $K$ cost-sensitive binary classification problems where $\alpha=1-c$ to verify Condition~\eqref{app:cond1}, that is, Condition~\eqref{app:cond1} is true if and only if $f^{*,y}_{1-c}=-1$ for all $y\in\mathcal{Y}$.

Next, we show that if Condition~\eqref{app:cond1} is false based on learning $K$ cost-sensitive binary classifiers, then it is sufficient to verify $ \argmax_y \eta_y(\x)$. This is because of the rejection cost $c$ is less $0.5$. 
Thus, if Condition~\eqref{app:cond1} is false, the only possibility is that there must exists one $y'$ such that $\eta_y'(\x) > 1-c$. The reason it can have at most one $y'$ to have $f^{*,y'}_{1-c}=1$ is because it indicates that $\eta_y' > 0.5$. Thus, the following rule can mimic Chow's rule.

\begin{equation*}
f^*(\x)=
\begin{cases}
\textrm{\rej} & \max_y f^{*,y}_{1-c}(\x) = -1 \text{,}\\
\argmax_y f^{*,y}_{1-c}(\x) & \rm{otherwise.}
\end{cases}
\end{equation*}
This concludes the proof.

\textbf{Remark:} we note that if Condition~\eqref{app:cond1} is true, it may not be possible to know $\argmax_y \eta_y(\x)$ given $K$ optimal cost-sensitive binary classifiers in general. 
However, in classification with rejection, it is not important to know $\argmax_y \eta_y(\x)$ if Condition~\eqref{app:cond1} is true since a classifier will refrain from making a prediction.

\label{proof:prop-chow-mult}
\subsection{Proof of Theorem~\ref{thm:calib}}
Let $g^*$ be a conditional risk minimizer that minimizes the pointwise conditional surrogate risk:
\begin{align*}
    g^*(\x) = \argmin_g W_{\mathcal{L}^{c,{\phi}}_\mathrm{CS}} \big(\vecg(\x); \veceta(\x)))
\end{align*}
Recall that the cost-sensitive surrogate loss is defined as
\begin{align*}
        \mathcal{L}^{c,\phi}_\mathrm{CS}(\vecg; \x, y) &= c \phi \big( g_{y}(\x)\big) + (1-c) \sum_{y' \neq y} \phi \big( -g_{y'}(\x) \big).
\end{align*}

Thus,

\begin{align*}
    W_{\mathcal{L}^{c,{\phi}}_\mathrm{CS}} \big(\vecg(\x); \veceta(\x))) &= \sum_{y \in \mathcal{Y}} \eta_y(\x) \mathcal{L}^{c,{\phi}}_\mathrm{CS} \big( \vecg; \x, y \big) \\
    &= \sum_{y \in \mathcal{Y}} \eta_y(\x) \left[ c \phi \big( g_{y}(\x)\big) + (1-c) \sum_{y' \neq y} \phi \big( -g_{y'}(\x) \big) \right]. \numberthis \label{eq:one}
\end{align*}
We can rewrite Eq.~\eqref{eq:one} as follows based on the perspective of $g_y$:
\begin{align*}
    W_{\mathcal{L}^{c,{\phi}}_\mathrm{CS}} \big(\vecg(\x); \veceta(\x))) 
    &= \sum_{y \in \mathcal{Y}} \eta_y(\x) \left[ c \phi \big( g_{y}(\x)\big) + (1-c) \sum_{y' \neq y} \phi \big( -g_{y'}(\x) \big) \right]. \\
    &= \sum_{y \in \mathcal{Y}} \left[ \eta_y(\x) c \phi \big( g_{y}(\x)\big) + (1-\eta_y(\x)) (1-c) \phi \big( -g_{y}(\x) \big) \right].
\end{align*}

It can be seen that $\eta_y(\x) c \phi \big( g_{y}(\x)\big) + (1-\eta_y(\x)) (1-c) \phi \big( -g_{y}(\x) \big)$ is a pointwise conditional risk of a cost-sensitive binary classifier $g_{y}$.
Thus, minimizing $W_{\mathcal{L}^{c,{\phi}}_\mathrm{CS}}$ can be viewed as independently minimizing the pointwise conditional risk in cost-sensitive binary classification for each $g_{y}$.
Thus $g^*_{y}$ corresponds to the conditional risk minimizer of the cost-sensitive binary classification where $y$ is a positive class and $y' \neq y$ is a negative class.

Recall the definition of $f(\x; \vecg^*)$:
\begin{align*}
f(\x; \vecg^*)=
\begin{cases}
\textrm{\rej} & \max_y g^*_y(\x) \leq 0 \text{,}\\
\textrm{\rej} & \makecell[l]{\exists y,y' \, \text{s.t.} \, y\neq y' \\ \text{and} \, g^*_y(\x), g^*_{y'}(\x) > 0\text{,}}\\
\argmax_y g^*_y(\x) & \rm{otherwise}
\end{cases}
\end{align*}
and $f^*(\x)$:
\begin{equation*}
f^*(\x)=
\begin{cases}
\textrm{\rej} & \max_y \eta_y(\x) \leq 1-c \text{,}\\
\argmax_y \eta_y(\x) & \text{otherwise.}
\end{cases}
\end{equation*}

First, we prove that $f(\x;\vecg^*)=f^*(\x)$ if $\phi$ is classification-calibrated.

The proof is based on the definition of $\alpha$-classification calibration ($\alpha$-CC) proposed by~\citep{scott2012calibrated} and its relationship with ordinary classification calibration~\citep{bartlett2006}, which is equivalent to $0.5$-CC. 
More specifically, it is known that a margin loss $\phi$ must also be classification-calibrated if it is $\alpha$-CC, i.e., its conditional risk minimizer matches the Bayes optimal classifier of cost-sensitive binary classification when using the weighted risk minimization based on $\alpha$~\citep{scott2012calibrated}. 

For a classification-calibrated margin loss $\phi$,  $\text{sign}(g^*_y)$ matches the Bayes optimal solution of the cost-sensitive binary classification~\citep{scott2012calibrated}, that is, $g^*_y(\x) >0$ if $\eta_y(\x) > 1-c$ and $g^*_y(\x) <0$ otherwise.
Thus if $g^*$ is obtained, the condition 
\begin{align}
    \label{cond:conflict}
    \exists y,y' \, \text{s.t.} \, y\neq y'  \text{and} \, g^*_y(\x), g^*_{y'}(\x) > 0
\end{align}
is impossible to occur since it is impossible to have $g^*_y(\x)>0$ and  $g^*_{y'}>0$ simultaneously (because that can occur only if $1-c < 0.5$ which is impossible if $c <0.5$.). 
Thus, it suffices to look at
\begin{align*}
f(\x; \vecg^*)=
\begin{cases}
\textrm{\rej} & \max_y g^*_y(\x) \leq 0 \text{,}\\
\argmax_y g^*_y(\x) & \rm{otherwise.}
\end{cases}
\end{align*}
$\max_y g^*_y(\x) \leq 0$ indicates that $\eta_y(\x) \leq 1-c$ for all $y\in\mathcal{Y}$ and thus coincides with the rejection criterion of Chow's rule.
On the other hand, if $\max_y g^*_y(\x) > 0$, then there exists only one $y$ such that $g^*_y(\x) > 0$ and $\eta_y(\x) > 1-c$. 
Thus, $f(\x; \vecg^*) = f^*(\x)$ if $\phi$ is classification-calibrated.

Next, we prove the converse of our theorem, that is, if a margin loss $\phi$ is not classification-calibrated, then there must exist the case where $\vecg^{\text{*w}}$ disagrees with Chow's rule, where $g^{\text{*w}}$ denotes an optimal solution for $\phi$ that is not classification-calibrated.

Here, let us drop $\x$ and only concerns $\veceta$, which is a probability simplex for simplicity, which suffices to prove our statement. 
If a margin loss $\phi$ is not classification-calibrated, all $\text{sign}(g_y^\text{*w})$ does not match the Bayes-optimal solution of the cost-sensitive classifier with respect to the rejection cost $c$, which suggests that there exists $\veceta$ that makes at least one $g^\text{*w}_y(\x)$ has the wrong sign compared with $g^*_y(\x)$.

We divide our analysis of $\veceta$ into two cases. 
First, we analyze the case where Chow's rule suggests to predict the input with the most probable class, which is the case where $\max_y \veceta_y > 1-c$.
Second, we analyze the case where Chow's rule suggests to refrain from making a prediction, i.e., $\max_y \veceta_y \leq 1-c$.  
Note that both cases cover all possibilities of $\veceta$. 
Moreover, we note that if $f(\x;g^\text{*w})$ disagrees with Chow's rule in at least one of the cases, it suffices to prove that $g^\text{*w}$ does not achieve calibration in classification with rejection.

\textbf{Case 1:} $\max_y \eta_y > 1-c$

\noindent In this case, Chow's rule suggests to accept and predict the most probable class $\argmax_y \eta_y$.

It is straightforward to see that for a decision rule $f(\x; \vecg)$ to match Chow's rule in this case, the sign of all $g_y$ must match the Bayes optimal classifier of binary cost-sensitive classification $g^*_y$, that is $g_y > 0$ only for $\eta_y > 1-c$ and $g_y' < 0$ for other less probable classes.
Based on $f(\x; \vecg)$, this is the only possible configuration of $\vecg$ to have the same decision as Chow's rule in the case where $\max_y \eta_y > 1-c$, which is predicting the most probable class $\argmax_y \eta_y$.

Thus, if the disagreement between of $\vecg^\text{*w}$ and $\vecg^*$ arises in the case where $\max_y \eta_y > 1-c$, $\vecg^\text{*w}$ must lead either predicting the wrong class (i.e., not the most probable class) or refrain from making a prediction, which both cases disagree with Chow's rule and lead to higher zero-one-$c$ risk.
Thus, it suffices to show that if $\phi$ is not classification-calibrated and the disagreement occurs when $\max_y \eta_y > 1-c$, then $\mathcal{L}_{\text{CS}}^{c,\phi}$ is not calibrated.

\textbf{Case 2:} $\max_y \veceta_y \leq 1-c$

In this case, Chow's rule suggests to refrain from making a prediction.

We will show that if $\vecg^{\text{*w}}$ agrees with Chow's rule for all $\veceta$ that lie in Case $1$, there must exist $\veceta$ in Case $2$ such that the decision of $\vecg^{*w}$ disagrees with Chow's rule if no further restriction such as the number of classes is imposed.

If $f(\x; \vecg^{*w}) = f^*(\x)$ everywhere in Case $1$, then it is guaranteed that $g^{*w}_y < 0$ for $\eta_y < c$, that is, $g^{*w}_y$ must predict negative correctly when $\max_y \eta_y > 1-c$ and only one $g^{*w}_y > 0$ for $\eta_y > 1-c$ for all $y \in \mathcal{Y}$. 
We will make use of these conditions to show that $f(\x; \vecg^{*w})$ will wrongly accept the data while Chow's rule suggests to refrain from making a prediction. 

In this case where $\max_y \veceta_y \leq 1-c$, there must exist $\veceta$ such that at least one $y$, we have $g_y > 0$ although $\eta_y \leq 1-c$ to make the sign of $g^{*w}$ disagrees with $g^*$.
Let $\beta \leq 1-c$ be a value that makes $g_y > 0$. 
There exists $\veceta$ such that $\eta_y = \beta$ for one $y$ and $\eta_{y'} \leq c$ for all other classes. Since $c>0$, it is always possible to make $\sum_{y\in\mathcal{Y}} \eta_y =1$ with a sufficient number of classes. 
Thus, we have $g_y > 0$ where $\eta_y = \beta$ and $g_{y'} < 0$ for all other classes.
This makes the decision of $f(\x; \vecg^{*w})$ to accept and predict the data while Chow's rule suggests to refrain from making a prediction, which means $\mathcal{L}_{\text{CS}}^{c,\phi}$ is not calibrated.

In summary, if the disagreement occurs in Case $1$ , it suffices to say $\mathcal{L}_{\text{CS}}^{c,\phi}$ is not calibrated. If disagreement does not occur in Case $1$, there exists $\veceta$ that makes $f(\x; \vecg^{*w})$ disagrees with Chow's rule. 
This concludes the proof of the converse case of our theorem.

As a result, the surrogate loss $\mathcal{L}^{c,{\phi}}_\mathrm{CS}$ is calibrated for classification with rejection, that is, $f(\x;\vecg^*)=f^*(\x)$ for all $\x \in \mathcal{X}$, if and only if $\phi$ is classification-calibrated. 
This concludes the proof.
\label{proof:calib}
\subsection{Proof of Theorem~\ref{thm:excess}}
\label{proof:excess}
To prove that
\begin{align*}
    R^{\zeroonedloss}(f)-R^
    {\zeroonedloss,*} &\leq  \sum_{i=1}^K \psi_{\phi,1-c}^{-1}(R^{\phi,i}_{1-c}(g_i)-R_{1-c}^{\phi,i,*}),
\end{align*}
we divide the proof into two steps. 
The first step is to prove that
\begin{align}
    R^{\zeroonedloss}(f)-R^
    {\zeroonedloss,*} &\leq  R^{ \mathcal{L}^{c,{\zerooneloss}}_\mathrm{CS}}(\vecg)-R^{ \mathcal{L}^{c,{\zerooneloss}}_\mathrm{CS},*}
    \label{ineq:first_step}
\end{align}
and the second step is to prove that
\begin{align}
R^{ \mathcal{L}^{c,{\zerooneloss}}_\mathrm{CS}}(\vecg)-R^{ \mathcal{L}^{c,{\zerooneloss}}_\mathrm{CS},*}  \leq \sum_{i=1}^K \psi_{\phi,1-c}^{-1}(R^{\phi,i}_{1-c}(g_i)-R_{1-c}^{\phi,i,*}).
\label{ineq:second_step}
\end{align}

\textbf{Proof of Ineq.~\eqref{ineq:first_step}:}

To prove this inequality, it suffices to prove that $R^{\zeroonedloss}(f) \leq R^{ \mathcal{L}^{c,{\zerooneloss}}_\mathrm{CS}}(\vecg)$ and $R^{\zeroonedloss,*}= R^{ \mathcal{L}^{c,{\zerooneloss}}_\mathrm{CS},*}$.

To prove that $R^{\zeroonedloss}(f) \leq R^{ \mathcal{L}^{c,{\zerooneloss}}_\mathrm{CS}}(\vecg)$, it suffices to show that 
\begin{align*}
    \zeroonedloss(f(\x;\vecg), y) \leq \mathcal{L}^{c,{\zerooneloss}}_\mathrm{CS}(\vecg; \x, y)
\end{align*}
holds for any choices of $\x \in \mathcal{X}$ and $y \in \mathcal{Y}$.
Thanks to the discrete nature of both the zero-one-$c$ loss $\zeroonedloss$ and the zero-one loss $\zerooneloss$, case analysis can be applied.

\textbf{Case 1: $\zeroonedloss(f(\x;\vecg), y) = 0$}

In this case, it suggests that $f(\x;\vecg)$ predicts a label that matches a label $y$. 
This is possible only if $g_y > 0$ and $g_{y'} <0$ for $y' \neq y$.
Recall the definition of the cost-sensitive surrogate loss:
\begin{align}
    \label{eq:surr-app}
    \mathcal{L}^{c,\zerooneloss}_\mathrm{CS}(\vecg; \x, y) = c \zerooneloss \big( g_{y}(\x)\big) + (1-c) \sum_{y' \neq y} \zerooneloss \big( -g_{y'}(\x) \big).
\end{align}
It can be seen that $ \mathcal{L}^{c,\zerooneloss}_\mathrm{CS}$ can only be larger or equal to zero, that is , $\mathcal{L}^{c,\zerooneloss}_\mathrm{CS}(\vecg; \x, y) \geq 0$.
Thus, $\zeroonedloss(f(\x;\vecg), y) \leq \mathcal{L}^{c,{\zerooneloss}}_\mathrm{CS}(\vecg; \x, y)$ if $\zeroonedloss(f(\x;\vecg), y) = 0$. 
Nevertheless, we can show that they are in fact equal, i.e., $\mathcal{L}^{c,{\zerooneloss}}_\mathrm{CS}(\vecg; \x, y) = 0$. 
This is because $c\zerooneloss \big( g_{y}(\x)\big) = 0$ and $(1-c)\zerooneloss \big(-g_{y'}(\x)\big) = 0$. 
Thus, $\mathcal{L}^{c,{\zerooneloss}}_\mathrm{CS}(\vecg; \x, y) = 0$ according to Eq.~\eqref{eq:surr-app}.
Therefore, we have
\begin{align*}
    \zeroonedloss(f(\x;\vecg), y) = \mathcal{L}^{c,{\zerooneloss}}_\mathrm{CS}(\vecg; \x, y) = 0.
\end{align*}

\textbf{Case 2: $\zeroonedloss(f(\x;\vecg), y) = 1$}
In this case, it can be shown that
\begin{align*}
    \zeroonedloss(f(\x;\vecg), y) = \mathcal{L}^{c,{\zerooneloss}}_\mathrm{CS}(\vecg; \x, y) = 1.
\end{align*}
It can be seen that $f(\x;\vecg)$ wrongly predicts the label, which only occurs when $g_{y'}(\x) > 0$ where $y' \neq y$ and $g_{y''}(\x) < 0$ for $y'' \neq y'$, which also includes the correct label.
Therefore, $c\zerooneloss \big( g_{y}(\x)\big) = c$,  $(1-c)\zerooneloss \big(-g_{y'}(\x)\big) = 1-c$, and $(1-c)\zerooneloss \big(-g_{y''}(\x)\big) = 0$ in the case where $y'' \neq y$ and $y'' \neq y'$.
As a result, the sum of the penalty becomes $c + (1-c) = 1$, which makes $\zeroonedloss(f(\x;\vecg), y) = \mathcal{L}^{c,{\zerooneloss}}_\mathrm{CS}(\vecg; \x, y) = 1$.

\textbf{Case 3: $\zeroonedloss(f(\x;\vecg), y) = c$}

Unlike the previous two cases, the bound can be loose in the case where $f(\x;\vecg)$ decides to refrain from making a prediction.
$f(\x;\vecg)=c$ indicates that a decision rule decides to reject. This is possible only if $g^*_y(\x) < 0$ for all $y \in \mathcal{Y}$ or Condition~\eqref{cond:conflict} holds.

If $g^*_y(\x) < 0$ for all $y \in \mathcal{Y}$, then $\mathcal{L}^{c,{\zerooneloss}}_\mathrm{CS}(\vecg^*; \x, y) = c$ because  $c\zerooneloss \big( g_{y}(\x)\big) = c$ and $(1-c)\zerooneloss \big(-g_{y'}(\x)\big) = 0$. 

If Condition~\eqref{cond:conflict} is true, we can show that $\mathcal{L}^{c,{\zerooneloss}}_\mathrm{CS}(\vecg^*; \x, y) \geq c$.
We will show by using the fact that the minimum possible value of $\mathcal{L}^{c,{\zerooneloss}}_\mathrm{CS}(\vecg^*; \x, y)$ when having a conflict is $1-c$. 
This is when there exists $g_y(\x) > 0$ and $g_{y'}(\x)> 0$, where $y$ is a correct label and $y'$ is a wrong label.
In this case, $c\zerooneloss \big( g_{y}(\x)\big) = 0$,  $(1-c)\zerooneloss \big( -g_{y'}(\x)\big) = 1-c$ and $(1-c)\zerooneloss \big( -g_{y''}(\x)\big) = 0$ for $y'' \neq y'$ and $y'' \neq y$.
If the conflicts of only two classifiers occur and both give the wrong labels, then we the penalty is $2(1-c)+c$. More conflicts only gain the higher penalty or nothing (if the conflict comes from the correct class), thus $\mathcal{L}^{c,{\zerooneloss}}_\mathrm{CS}(\vecg^*; \x, y) \geq 1-c \geq c$.

Therefore, $R^{\zeroonedloss}(f) \leq R^{ \mathcal{L}^{c,{\zerooneloss}}_\mathrm{CS}}(\vecg)$.

Next, to prove that $R^{ \mathcal{L}^{c,{\zerooneloss}}_\mathrm{CS},*}= R^{ \mathcal{L}^{c,{\zerooneloss}}_\mathrm{CS},*}$, it suffices to show that $\zeroonedloss(f(\x;\vecg^*), y) = \mathcal{L}^{c,{\zerooneloss}}_\mathrm{CS}(\vecg^*; \x, y)$ for any choices of $\x \in \mathcal{X}$ and $y \in \mathcal{Y}$.

\textbf{Case 1: $\zeroonedloss(f(\x;\vecg^*), y) = 0$}

From the previous analysis, in this case we have
\begin{align*}
    \zeroonedloss(f(\x;\vecg), y) = \mathcal{L}^{c,{\zerooneloss}}_\mathrm{CS}(\vecg; \x, y) = 0,
\end{align*}
for any $\vecg$. 
Therefore, it must also hold when $\vecg=\vecg^*$.

\textbf{Case 2: $\zeroonedloss(f(\x;\vecg^*), y) = 1$}

From the previous analysis, in this case we have
\begin{align*}
    \zeroonedloss(f(\x;\vecg), y) = \mathcal{L}^{c,{\zerooneloss}}_\mathrm{CS}(\vecg; \x, y) = 1,
\end{align*}
for any $\vecg$. 
Therefore, it must also hold when $\vecg=\vecg^*$.

\textbf{Case 3: $\zeroonedloss(f(\x;\vecg^*), y) = c$}
As suggested in the proof of Theorem~\ref{thm:calib} that Condition~\eqref{cond:conflict} is impossible to occur for $f(\x;\vecg^*)$. 
Thus, if $\zeroonedloss(f(\x;\vecg^*)$ rejects, it means that $g^*_y(\x) < 0$ for all $y \in \mathcal{Y}$. 
This makes $\mathcal{L}^{c,{\zerooneloss}}_\mathrm{CS}(\vecg^*; \x, y) = c$ because  $c\zerooneloss \big( g_{y}(\x)\big) = c$ and $(1-c)\zerooneloss \big(-g_{y'}(\x)\big) = 0$.

Since $\zeroonedloss(f(\x;\vecg^*), y) = \mathcal{L}^{c,{\zerooneloss}}_\mathrm{CS}(\vecg^*; \x, y)$ always holds, $R^{ \zeroonedloss,*}= R^{ \mathcal{L}^{c,{\zerooneloss}}_\mathrm{CS},*}$.

We have proven that $R^{\zeroonedloss}(f) \leq R^{ \mathcal{L}^{c,{\zerooneloss}}_\mathrm{CS}}(\vecg)$ and $R^{ \zeroonedloss,*}= R^{ \mathcal{L}^{c,{\zerooneloss}}_\mathrm{CS},*}$. 
Thus, Ineq.~\eqref{ineq:first_step} holds.

\textbf{Proof of Ineq.~\eqref{ineq:second_step}:}

In this part, the proof no longer involves with classification with rejection but the well-studied cost-sensitive classification.
We borrow the existing result by~\citet{scott2012calibrated} to prove this part.

Recall the following pointwise conditional risk:
\begin{align*}
    W_{\mathcal{L}^{c,{\phi}}_\mathrm{CS}} \big(\vecg(\x); \veceta(\x))) 
    &= \sum_{y \in \mathcal{Y}} \eta_y(\x) \left[ c \phi \big( g_{y}(\x)\big) + (1-c) \sum_{y' \neq y} \phi \big( -g_{y'}(\x) \big) \right]. \\
    &= \sum_{y \in \mathcal{Y}} \left[ \eta_y(\x) c \phi \big( g_{y}(\x)\big) + (1-\eta_y(\x)) (1-c) \phi \big( -g_{y}(\x) \big) \right]
\end{align*}
also holds when $\phi=\zerooneloss$. 
This suggests that the risk of the cost-sensitive surrogate equals to the sum of the pointwise surrogate risks of $K$ cost-sensitive binary classification problem, i.e., $R^{ \mathcal{L}^{c,{\zerooneloss}}_\mathrm{CS}}(\vecg) = \sum_{i=1}^K R^{\zerooneloss,i}_{1-c}(g_i)$ for any $\vecg$. 
Thus, we have
\begin{align*}
    R^{ \mathcal{L}^{c,{\zerooneloss}}_\mathrm{CS}}(\vecg)-R^{ \mathcal{L}^{c,{\zerooneloss}}_\mathrm{CS},*}  = \sum_{i=1}^K R^{\zerooneloss,i}_{1-c}(g_i)-R_{1-c}^{\zerooneloss,i,*}.
\end{align*}
Next, since $\phi$ is classification-calibrated, \citet{scott2012calibrated} proved that there exists $\psi_{\phi,1-c}\colon \R\to\R$, which is a non-decreasing invertible function and $\psi_{\phi,1-c}(0)=0$ such that
\begin{align*}
    R^{\zerooneloss,i}_{1-c}(g_i)-R_{1-c}^{\zerooneloss,i,*} \leq \psi_{\phi,1-c}^{-1}(R^{\phi,i}_{1-c}(g_i)-R_{1-c}^{\phi,i,*})
\end{align*}
By adding the excess risk bound of $K$ cost-sensitive binary classification problems, we have
\begin{align*}
    R^{\zeroonedloss}(f)-R^
    {\zeroonedloss,*}
    &\leq \sum_{i=1}^K \psi_{\phi,1-c}^{-1}(R^{\phi,i}_{1-c}(g_i)-R_{1-c}^{\phi,i,*}).
\end{align*}
Thus, Ineq.~\eqref{ineq:second_step} holds.

Since Ineq.~\eqref{ineq:first_step} and Ineq.~\eqref{ineq:second_step} hold, we concludes the proof of the excess risk bound of classification with rejection based on cost-sensitive classification for a general classification-calibrated loss.
\clearpage
\section{Experiment Details
\label{sec:exp-detail}
}
In this section, we provide more information on datasets and implementation details.
\subsection{Datasets}
We used the train and test data for MNIST, Fashion-MNIST, and KMNIST for training and testing, respectively.
For MNIST, Fashion-MNIST, and KMNIST, we used the same test data as the one provided for testing for those datasets.
For hyperparameter selection, we split ten percent of the training data to use as validation data (i.e., $6000$). For the other ninety percent of training data, they were used for training all methods. 
Note that only ANGLE and SCE used validation data. 

For the datasets other than MNIST, Fashion-MNIST, and KMNIST, we randomly used half of the dataset for training all methods.
For clean-labeled and noisy-labeled classification, we used ten percent of all data for validation and the other fourty percent were used for testing. 
For PU-classification, we used twenty percent of all data for validation and the other thirty percent were used for testing.  
Again, note that only ANGLE and SCE used validation data. 

Table~\ref{table:dataset} shows the specification of the benchmark datasets used in this paper. 
Subj was preprocessed by using $100$-dimensional GloVe mean word embedding~\citep{pennington2014glove}.

\begin{table}[t]%[ht]
  \begin{center}
	\caption{Specification of benchmark datasets: the number of features, the number of classes, the number of data.} \label{table:dataset}
	\vspace{0.1in}
	\begin{tabular}{|c|c|c|c|} \hline
Name     & $\#$features & $\#$classes & $\#$data \\ \hline
Gisette  & 5000           & 2           & 7000          \\
Phishing & 30           & 2           & 11050        \\
Spambase  & 57           & 2           & 4601        \\ 
Subj   & 100           & 2          & 10000       \\ 
Twonorm   & 20           & 2          & 7400    \\
Gas-Drift   & 128           & 6          & 13910         \\ 
HAR   & 561           & 6          & 10299        \\ 
MNIST   & 28$\times$28           & 10          & 70000        \\ 
Fashion-MNIST   & 28$\times$28             & 10          & 70000        \\ 
KMNIST   & 28$\times$28             & 10          & 70000      \\ 
	  \hline
	\end{tabular}
  \end{center}
\end{table}
\subsection{Implementation details}
We used a linear-in-input model for all binary classification datasets. 
For MNIST, Fashion-MNIST, KMNIST, we used the same convolutional neural network (CNN) architecture.
The CNN model consists of a sequence of two convolutional layers with $32$ channels and two convolutional layers with $64$ channels, followed by a max pooling layer and two linear layers with dimension $128$. 
The kernel size of convolutional layers is $3$, and the kernel size of max pooling layer is $2$.
Dropout with probability $0.5$ is used between two linear layers.
We used rectifier linear units (ReLU) \cite{nair2010rectified} as the non-linear activation function.
For HAR and Gas-Drift, we used the one hidden layer multilayer perceptron as a model ($d-64-1$). 
We also applied batch normalization~\citep{ioffe2015batch} at the final layer to stabilize training. 
The objective functions were optimized using Adam~\citep{adam}. 
The experiment code for implementing a model was written using PyTorch~\citep{pytorch2019}.
We ran 10 trials for each experiment. 

\subsubsection{Clean-labeled and noisy-labeled classification}
\paragraph{Data generation process}
For noise labels, the noise rate was $0.25$, i.e., $25\%$ of labels are randomly flipped into other classes.
No data augmentation was used for all experiments.

\paragraph{Hyperparameters}
For all experiments, learning rate was set to $0.001$, batch size was $256$.
The model was trained for $10$ epochs for the convolutional neural networks and $100$ epochs for both the linear-in-parameter model and multilayer perceptron.

\subsubsection{PU-classification}
\paragraph{Problem setting}
PU-classification considers situations when only positive and unlabeled data are available.
We denote class conditional densities by $p_+(\x)=p(\x|y=+1)$ and $p_-(\x)=p(\x|y=-1)$ and the class prior probability by $\pi=p(y=+1)$.
Then the distribution for unlabeled data can be expressed as 
\begin{align*}
    p(\x)=\pi p_+(\x) + (1-\pi)p_-(\x).
\end{align*}

Let $\mathcal{L}(\vecg;\x,y)$ be a loss function.
It is known that the expected classification risk can be expressed as~\citep{du2015convex}:
\begin{align*}
      \EJoint[\mathcal{L}(\vecg;\x,y)] =  \pi\mathbb{E}_{\x\sim p_+(\x)}\left[\mathcal{L}(\vecg;\x,+1)\right] -\pi\mathbb{E}_{\x\sim p_+(\x)}\left[\mathcal{L}(\vecg;\x,-1)\right]+\mathbb{E}_{\x\sim p(\x)}\left[\mathcal{L}(\vecg;\x,-1)\right].
\end{align*}

Then, the unbiased risk estimator given positive examples $\{\x^\mathrm{p}_i\}_{i=1}^{n_\mathrm{p}} \stackrel{\mathrm{i.i.d.}}{\sim} p_+(\x)$ and unlabeled examples $\{\x^\mathrm{u}_j\}_{j=1}^{n_\mathrm{u}} \stackrel{\mathrm{i.i.d.}}{\sim} p(\x)$ can be expressed as
\begin{align}
    \label{eq:unbiased}
      \frac{\pi}{n_\mathrm{p}}\sum_{i=1}^{n_\mathrm{p}}\left[\mathcal{L}(\vecg;\x,+1)\right] -\frac{\pi}{n_\mathrm{p}}\sum_{i=1}^{n_\mathrm{p}}\left[\mathcal{L}(\vecg;\x,-1)\right]+\frac{1}{n_\mathrm{u}}\sum_{j=1}^{n_\mathrm{u}}\left[\mathcal{L}(\vecg;\x,-1)\right].
\end{align}
However, \citet{kiryo2017} suggested that Eq.~\eqref{eq:unbiased} is prone to overfitting because $\vecg$ may treat all unlabeled data as negative to minimize the empirical risk.
As a result, it was suggested to instead minimize the following empirical risk:
\begin{align}
\label{eq:biased}
 \hat{R}_{\mathrm{PU}}^\mathcal{L}(\vecg)= \frac{\pi}{n_\mathrm{p}}\sum_{i=1}^{n_\mathrm{p}}\left[\mathcal{L}(\vecg;\x,+1)\right] +\max\left(0,   \frac{1}{n_\mathrm{u}}\sum_{j=1}^{n_\mathrm{u}}\left[\mathcal{L}(\vecg;\x,-1)\right] - \frac{\pi}{n_\mathrm{p}}\sum_{i=1}^{n_\mathrm{p}}\left[\mathcal{L}(\vecg;\x,-1)\right]\right),
\end{align}
which is known to be a biased but still consistent estimator.
With Eq.~\eqref{eq:biased}, the choice of loss functions is flexible and we can easily apply any loss function to learn a classifier with rejection, e.g., our cost-sensitive approach can be easily applied in PU-classification by minimizing $\hat{R}_{\mathrm{PU}}^{\mathcal{L}_{\mathrm{CS}}^{c,\phi}}(\vecg)$, which equals to solving two cost-sensitive PU-classification.
We refer the readers to~\citet{du2014,du2015convex,kiryo2017} for more detail about PU-classification and~\citet{charoenphakdee2019positive} for more detail about cost-sensitive PU-classification.

\paragraph{Data generation process}
PU-classification needs two set of data: positive data and unlabeled data. We first decided the size of unlabeled data to be around the size of original training data, but truncated to be divisible by $200$.
Then, the size of positive data was set to be $\frac15$ of the size of unlabeled data.

After deciding the size of two sets, we than sampled from the original training data.
The positive data for PU is sampled from the original positive training data without replacement.
Then, from the left positive training data and negative data, we sampled without replacement for unlabeled data according to the value of class prior.

\paragraph{Hyperparameters}
The class prior for the positive class is set to be $0.7$ throughout all PU-classification experiments.
Learning rate was set to $0.001$, batch size was 64, and the number of epochs was $100$.

\section{Additional Experiment Results
\label{sec:add-exp}
}
In this section, we report full experimental results in a table format with varying rejection costs for clean-labeled classification, noisy-labeled classification, and PU-classification, respectively. 
We also provide more discussion on our proposed rejection conditions, i.e., Conds.~\eqref{eq:amb-reject} and \eqref{eq:morethanone-reject}.

\subsection{Full experimental results in a table format}
We report the mean and standard error over ten trials of the test empirical zero-one-$c$ risk, rejection ratio, and test error on the non-rejected data.
\subsubsection*{Table index:}
\begin{itemize}
    \item Tables~\ref{app:clean-bin-01c}, \ref{app:clean-bin-err} and \ref{app:clean-bin-rej}: Test empirical $0\text{-}1\text{-}c$ risk, classification error on accepted data, and rejection ratio for clean-labeled binary classification with rejection, respectively.
    \item Tables~\ref{app:noisy-bin-01c}, \ref{app:noisy-bin-err} and \ref{app:noisy-bin-rej}: Test empirical $0\text{-}1\text{-}c$ risk, classification error on accepted data, and rejection ratio for noisy-labeled binary classification with rejection, respectively.
    \item Tables~\ref{app:pu-bin-01c}, \ref{app:pu-bin-err} and \ref{app:pu-bin-rej}: Test empirical $0\text{-}1\text{-}c$ risk, classification error on accepted data, and rejection ratio for positive-unlabeled binary classification with rejection, respectively.
    \item Tables~\ref{app:clean-mult-01c}, \ref{app:clean-mult-err} and \ref{app:clean-mult-rej}: Test empirical $0\text{-}1\text{-}c$ risk, classification error on accepted data, and rejection ratio for clean-labeled multiclass classification with rejection, respectively.
    \item Tables~\ref{app:noisy-mult-01c}, \ref{app:noisy-mult-err} and \ref{app:noisy-mult-rej}: Test empirical $0\text{-}1\text{-}c$ risk, classification error on accepted data, and rejection ratio for noisy-labeled multiclass classification with rejection, respectively.
\end{itemize}

\begin{table*}[t]
    \centering
    \caption{Mean and standard error of $0\text{-}1\text{-}c$ risk of the clean-labeled binary classification setting (rescaled to $0\text{-}100$).}
    \label{app:clean-bin-01c}
\begin{tabular}{lcccccc}
    \toprule
    Dataset & $c$
    & SCE & DEFER & ANGLE & CS-hinge & CS-sigmoid \\
    \midrule
\multirow{7}{*}{Gisette} & 0.10 & $2.20(0.10)$ & $5.33(0.16)$ & $2.41(0.16)$ & $2.47(0.14)$ & $2.35(0.09)$\\
 & 0.15 & $2.72(0.10)$ & $6.41(0.34)$ & $2.94(0.32)$ & $2.69(0.13)$ & $2.65(0.11)$\\
 & 0.20 & $3.10(0.09)$ & $7.43(0.22)$ & $3.25(0.23)$ & $2.82(0.13)$ & $2.78(0.09)$\\
 & 0.25 & $3.36(0.16)$ & $8.53(0.25)$ & $3.59(0.43)$ & $3.04(0.10)$ & $3.09(0.16)$\\
 & 0.30 & $3.72(0.20)$ & $9.68(0.30)$ & $3.69(0.22)$ & $3.19(0.16)$ & $3.15(0.11)$\\
 & 0.35 & $3.95(0.30)$ & $10.46(0.19)$ & $3.79(0.24)$ & $3.41(0.16)$ & $3.25(0.07)$\\
 & 0.40 & $4.27(0.17)$ & $11.31(0.29)$ & $3.70(0.13)$ & $3.52(0.15)$ & $3.38(0.21)$\\
\bottomrule
\end{tabular}
\end{table*}

\begin{table*}[t]
    \centering
    \caption{Mean and standard error of $0\text{-}1$ risk for accepted data of the clean-labeled binary classification setting (rescaled to $0\text{-}100$).}
    \label{app:clean-bin-err}
    \begin{tabular}{lcccccc}
    \toprule
    Dataset & $c$
    & SCE & DEFER & ANGLE & CS-hinge & CS-sigmoid \\
    \midrule
\multirow{7}{*}{Gisette} & 0.10 & $1.34(0.19)$ & $3.80(0.24)$ & $1.34(0.23)$ & $2.16(0.16)$ & $2.10(0.09)$\\
 & 0.15 & $1.60(0.33)$ & $3.58(0.41)$ & $1.73(0.68)$ & $2.21(0.15)$ & $2.32(0.13)$\\
 & 0.20 & $1.38(0.06)$ & $3.48(0.26)$ & $2.16(0.67)$ & $2.20(0.13)$ & $2.38(0.10)$\\
 & 0.25 & $1.67(0.33)$ & $3.43(0.33)$ & $2.35(0.83)$ & $2.27(0.12)$ & $2.66(0.18)$\\
 & 0.30 & $1.59(0.28)$ & $3.56(0.30)$ & $2.48(0.63)$ & $2.29(0.18)$ & $2.70(0.12)$\\
 & 0.35 & $1.81(0.37)$ & $3.39(0.19)$ & $3.29(0.51)$ & $2.32(0.11)$ & $2.81(0.07)$\\
 & 0.40 & $1.79(0.30)$ & $3.19(0.13)$ & $2.48(0.47)$ & $2.38(0.14)$ & $2.97(0.18)$\\
\bottomrule
\end{tabular}
\end{table*}

\begin{table*}[t]
    \centering
    \caption{Mean and standard error of rejection ratio of the clean-labeled binary classification setting (rescaled to $0\text{-}100$).}
    \label{app:clean-bin-rej}
    \begin{tabular}{lcccccc}
    \toprule
    Dataset & $c$
    & SCE & DEFER & ANGLE & CS-hinge & CS-sigmoid \\
    \midrule
\multirow{7}{*}{Gisette} & 0.10 & $9.88(1.63)$ & $24.67(0.66)$ & $12.25(3.40)$ & $3.97(0.34)$ & $3.22(0.15)$\\
 & 0.15 & $8.32(2.52)$ & $24.78(0.70)$ & $8.93(3.72)$ & $3.73(0.16)$ & $2.63(0.16)$\\
 & 0.20 & $9.20(0.28)$ & $23.91(0.59)$ & $5.97(3.66)$ & $3.51(0.28)$ & $2.26(0.14)$\\
 & 0.25 & $7.21(1.87)$ & $23.63(0.69)$ & $5.37(3.74)$ & $3.39(0.24)$ & $1.90(0.21)$\\
 & 0.30 & $7.50(1.36)$ & $23.12(0.39)$ & $4.37(2.37)$ & $3.24(0.17)$ & $1.65(0.16)$\\
 & 0.35 & $6.45(1.81)$ & $22.36(0.71)$ & $1.57(1.34)$ & $3.34(0.36)$ & $1.37(0.13)$\\
 & 0.40 & $6.48(1.02)$ & $22.08(0.88)$ & $3.24(1.24)$ & $3.03(0.18)$ & $1.13(0.13)$\\
\bottomrule
\end{tabular}
\end{table*}

\begin{table*}[t]
    \centering
    \caption{Mean and standard error of $0\text{-}1\text{-}c$ risk of the noisy-labeled binary classification setting (rescaled to $0\text{-}100$).}
    \label{app:noisy-bin-01c}
    \begin{tabular}{lcccccc}
    \toprule
    Dataset & $c$
    & SCE & DEFER & ANGLE & CS-hinge & CS-sigmoid \\
    \midrule
\multirow{7}{*}{Gisette} & 0.10 & $10.05(0.28)$ & $21.29(0.70)$ & $10.03(0.14)$ & $23.98(0.54)$ & $7.64(0.64)$\\
 & 0.15 & $14.40(0.45)$ & $22.97(0.42)$ & $14.89(0.61)$ & $26.03(0.54)$ & $9.09(0.60)$\\
 & 0.20 & $18.24(0.46)$ & $24.36(0.95)$ & $19.03(0.33)$ & $27.23(0.69)$ & $11.26(0.64)$\\
 & 0.25 & $22.30(0.81)$ & $26.03(0.74)$ & $22.60(0.82)$ & $27.76(0.93)$ & $12.76(0.76)$\\
 & 0.30 & $25.39(0.50)$ & $27.75(0.62)$ & $25.80(0.81)$ & $28.47(0.80)$ & $13.52(1.06)$\\
 & 0.35 & $28.26(1.12)$ & $29.18(0.63)$ & $27.69(0.84)$ & $29.00(0.79)$ & $14.91(0.70)$\\
 & 0.40 & $31.01(1.00)$ & $31.39(0.66)$ & $28.78(1.07)$ & $29.72(0.81)$ & $15.88(0.68)$\\
\bottomrule
\end{tabular}
\end{table*}

\begin{table*}[t]
    \centering
    \caption{Mean and standard error of $0\text{-}1$ risk for accepted data of the noisy-labeled binary classification setting (rescaled to $0\text{-}100$).}
    \label{app:noisy-bin-err}
    \begin{tabular}{lcccccc}
    \toprule
    Dataset & $c$
    & SCE & DEFER & ANGLE & CS-hinge & CS-sigmoid \\
    \midrule
\multirow{7}{*}{Gisette} & 0.10 & $10.46(2.00)$ & $27.33(1.04)$ & $19.21(21.10)$ & $26.18(0.67)$ & $6.99(0.81)$\\
 & 0.15 & $12.35(1.86)$ & $27.30(0.74)$ & $12.58(3.53)$ & $27.47(0.62)$ & $7.70(0.69)$\\
 & 0.20 & $11.51(1.58)$ & $26.77(1.47)$ & $16.13(2.20)$ & $28.05(0.76)$ & $9.33(0.78)$\\
 & 0.25 & $12.87(1.78)$ & $26.61(1.15)$ & $17.27(3.11)$ & $28.04(1.03)$ & $10.28(0.81)$\\
 & 0.30 & $13.11(1.15)$ & $26.50(0.96)$ & $20.47(4.24)$ & $28.32(0.88)$ & $10.71(1.16)$\\
 & 0.35 & $13.96(2.14)$ & $25.88(1.04)$ & $24.33(2.90)$ & $28.47(0.86)$ & $11.79(0.75)$\\
 & 0.40 & $15.22(1.47)$ & $26.31(1.17)$ & $25.84(3.05)$ & $28.90(0.85)$ & $12.80(0.66)$\\
\bottomrule
\end{tabular}
\end{table*}

\begin{table*}[t]
    \centering
    \caption{Mean and standard error of rejection ratio of the noisy-labeled binary classification setting (rescaled to $0\text{-}100$).
}
    \label{app:noisy-bin-rej}
    \begin{tabular}{lcccccc}
    \toprule
    Dataset & $c$
    & SCE & DEFER & ANGLE & CS-hinge & CS-sigmoid \\
    \midrule
\multirow{7}{*}{Gisette} & 0.10 & $84.89(5.15)$ & $34.86(0.94)$ & $96.31(3.59)$ & $13.58(0.60)$ & $21.90(1.23)$\\
 & 0.15 & $76.78(7.37)$ & $35.21(0.63)$ & $82.31(10.98)$ & $11.49(0.59)$ & $19.09(1.09)$\\
 & 0.20 & $78.23(7.98)$ & $35.62(0.71)$ & $70.45(8.83)$ & $10.17(0.39)$ & $18.17(1.26)$\\
 & 0.25 & $77.22(7.48)$ & $35.79(0.95)$ & $62.96(16.44)$ & $9.41(0.31)$ & $16.90(1.17)$\\
 & 0.30 & $72.54(4.04)$ & $35.76(0.39)$ & $47.33(19.82)$ & $8.84(0.58)$ & $14.59(1.32)$\\
 & 0.35 & $67.36(7.82)$ & $36.13(1.42)$ & $28.41(12.22)$ & $8.21(0.58)$ & $13.41(0.71)$\\
 & 0.40 & $63.43(5.69)$ & $37.05(1.03)$ & $18.55(11.06)$ & $7.41(0.33)$ & $11.32(0.83)$\\
\bottomrule
\end{tabular}
\end{table*}

\begin{table*}[t]
\centering
\caption{Mean and standard error of $0\text{-}1\text{-}c$ risk of the positive-unlabeled classification setting (rescaled to $0\text{-}100$).}
\label{app:pu-bin-01c}
\begin{tabular}{lcccccc}
\toprule
Dataset & $c$
& SCE & DEFER & ANGLE & CS-hinge & CS-sigmoid \\
\midrule
\multirow{7}{*}{Gisette} & 0.10 & $11.18(0.21)$ & $17.90(1.22)$ & $24.85(2.86)$ & $10.63(0.62)$ & $9.28(0.79)$\\
 & 0.15 & $16.01(0.14)$ & $20.96(1.36)$ & $30.54(1.93)$ & $15.31(1.20)$ & $11.17(0.62)$\\
 & 0.20 & $20.87(0.18)$ & $21.63(2.20)$ & $30.83(2.65)$ & $19.87(1.37)$ & $12.11(0.69)$\\
 & 0.25 & $25.59(0.26)$ & $22.98(1.84)$ & $34.17(2.15)$ & $22.73(1.32)$ & $14.23(0.98)$\\
 & 0.30 & $30.16(0.81)$ & $24.69(2.80)$ & $33.99(2.17)$ & $26.46(1.38)$ & $15.41(0.93)$\\
 & 0.35 & $34.67(0.91)$ & $26.10(1.91)$ & $36.58(2.46)$ & $29.17(1.90)$ & $18.10(1.17)$\\
 & 0.40 & $40.04(0.17)$ & $26.76(2.57)$ & $36.25(2.42)$ & $31.08(1.01)$ & $19.69(0.93)$\\
\bottomrule
\end{tabular}
\end{table*}

\begin{table*}[t]
\centering
\caption{Mean and standard error of $0\text{-}1$ risk for accepted data of the positive-unlabeled classification setting (rescaled to $0\text{-}100$).}
\label{app:pu-bin-err}
\begin{tabular}{lcccccc}
\toprule
Dataset & $c$
& SCE & DEFER & ANGLE & CS-hinge & CS-sigmoid \\
\midrule
\multirow{7}{*}{Gisette} & 0.10 & $45.39(9.08)$ & $21.52(1.57)$ & $29.98(2.93)$ & $11.04(0.98)$ & $9.03(1.06)$\\
 & 0.15 & $49.00(7.41)$ & $23.43(2.00)$ & $32.99(1.91)$ & $15.54(1.87)$ & $9.69(0.89)$\\
 & 0.20 & $46.40(7.76)$ & $22.30(3.05)$ & $32.12(2.69)$ & $19.82(2.05)$ & $9.35(1.01)$\\
 & 0.25 & $44.87(9.27)$ & $22.28(2.42)$ & $34.83(2.03)$ & $21.63(2.03)$ & $10.40(0.98)$\\
 & 0.30 & $39.79(10.70)$ & $23.05(3.51)$ & $34.13(2.23)$ & $24.77(2.12)$ & $10.47(1.07)$\\
 & 0.35 & $34.16(10.19)$ & $23.07(2.43)$ & $36.60(2.56)$ & $26.73(2.59)$ & $11.25(1.76)$\\
 & 0.40 & $41.78(4.27)$ & $23.30(2.76)$ & $36.17(2.47)$ & $27.11(1.75)$ & $11.53(0.98)$\\
\bottomrule
\end{tabular}
\end{table*}

\begin{table*}[t]
\centering
\caption{Mean and standard error of rejection ratio of the positive-unlabeled classification setting (rescaled to $0\text{-}100$).}
\label{app:pu-bin-rej}
\begin{tabular}{lcccccc}
\toprule
Dataset & $c$
& SCE & DEFER & ANGLE & CS-hinge & CS-sigmoid \\
\midrule
\multirow{7}{*}{Gisette} & 0.10 & $96.36(1.36)$ & $31.50(4.21)$ & $26.12(6.31)$ & $39.56(3.48)$ & $26.76(2.69)$\\
 & 0.15 & $96.92(0.66)$ & $29.06(3.77)$ & $13.71(4.00)$ & $35.44(3.75)$ & $27.78(1.75)$\\
 & 0.20 & $96.50(0.93)$ & $25.94(3.29)$ & $11.09(3.93)$ & $33.52(2.22)$ & $25.88(1.35)$\\
 & 0.25 & $96.39(1.54)$ & $26.53(3.65)$ & $7.14(2.78)$ & $31.95(2.10)$ & $26.34(3.03)$\\
 & 0.30 & $95.30(2.98)$ & $24.51(3.87)$ & $4.46(2.48)$ & $31.83(2.13)$ & $25.29(2.97)$\\
 & 0.35 & $93.86(3.04)$ & $25.63(3.14)$ & $3.55(1.98)$ & $29.38(2.52)$ & $28.79(2.32)$\\
 & 0.40 & $96.08(0.65)$ & $21.09(2.85)$ & $2.11(2.06)$ & $30.58(3.38)$ & $28.63(2.62)$\\
\bottomrule
\end{tabular}
\end{table*}

\begin{table*}[t]
    \centering
    \caption{Mean and standard error of $0\text{-}1\text{-}c$ risk of the clean-labeled multiclass classification setting (rescaled to $0\text{-}100$).}
    \label{app:clean-mult-01c}
    \begin{tabular}{lcccccc}
    \toprule
    Dataset & $c$
    & SCE & DEFER & ANGLE & CS-hinge & CS-sigmoid \\
    \midrule
\multirow{7}{*}{Gas-Drift} & 0.10 & $0.88(0.20)$ & $2.53(0.39)$ & $1.00(0.12)$ & $1.22(0.08)$ & $0.70(0.06)$\\
 & 0.15 & $1.07(0.40)$ & $2.84(0.50)$ & $1.13(0.19)$ & $1.39(0.13)$ & $0.70(0.07)$\\
 & 0.20 & $0.83(0.14)$ & $3.21(0.40)$ & $1.09(0.15)$ & $1.34(0.17)$ & $0.74(0.04)$\\
 & 0.25 & $0.89(0.25)$ & $3.49(0.51)$ & $1.12(0.15)$ & $1.27(0.22)$ & $0.85(0.10)$\\
 & 0.30 & $0.86(0.24)$ & $2.90(1.08)$ & $1.12(0.33)$ & $1.33(0.27)$ & $0.87(0.09)$\\
 & 0.35 & $1.13(0.76)$ & $2.64(0.59)$ & $1.10(0.27)$ & $1.24(0.27)$ & $0.91(0.12)$\\
 & 0.40 & $0.82(0.19)$ & $2.91(0.66)$ & $1.12(0.23)$ & $1.17(0.22)$ & $0.90(0.07)$\\
\bottomrule
\end{tabular}
\end{table*}

\begin{table*}[t]
    \centering
    \caption{Mean and standard error of $0\text{-}1$ risk for accepted data of the clean-labeled multiclass classification setting (rescaled to $0\text{-}100$).}
    \label{app:clean-mult-err}
    \begin{tabular}{lcccccc}
    \toprule
    Dataset & $c$
    & SCE & DEFER & ANGLE & CS-hinge & CS-sigmoid \\
    \midrule
\multirow{7}{*}{Gas-Drift} & 0.10 & $0.87(0.19)$ & $0.33(0.08)$ & $0.82(0.18)$ & $0.39(0.04)$ & $0.51(0.09)$\\
 & 0.15 & $1.07(0.40)$ & $0.39(0.09)$ & $0.88(0.16)$ & $0.40(0.05)$ & $0.46(0.06)$\\
 & 0.20 & $0.83(0.14)$ & $0.38(0.12)$ & $0.98(0.12)$ & $0.40(0.03)$ & $0.47(0.04)$\\
 & 0.25 & $0.88(0.25)$ & $0.42(0.13)$ & $1.01(0.17)$ & $0.44(0.05)$ & $0.56(0.08)$\\
 & 0.30 & $0.86(0.24)$ & $0.40(0.09)$ & $0.84(0.20)$ & $0.53(0.20)$ & $0.58(0.07)$\\
 & 0.35 & $1.13(0.76)$ & $0.41(0.09)$ & $0.95(0.13)$ & $0.61(0.17)$ & $0.56(0.07)$\\
 & 0.40 & $0.82(0.19)$ & $0.45(0.08)$ & $1.04(0.24)$ & $0.51(0.07)$ & $0.56(0.07)$\\
\bottomrule
\end{tabular}
\end{table*}

\begin{table*}[t]
    \centering
    \caption{Mean and standard error of rejection ratio of the clean-labeled multiclass classification setting (rescaled to $0\text{-}100$).}
    \label{app:clean-mult-rej}
    \begin{tabular}{lcccccc}
    \toprule
    Dataset & $c$
    & SCE & DEFER & ANGLE & CS-hinge & CS-sigmoid \\
    \midrule
\multirow{7}{*}{Gas-Drift} & 0.10 & $0.03(0.09)$ & $22.78(3.96)$ & $1.91(1.34)$ & $8.61(0.95)$ & $2.03(0.72)$\\
 & 0.15 & $0.00(0.00)$ & $16.79(3.17)$ & $1.77(0.99)$ & $6.77(0.94)$ & $1.68(0.46)$\\
 & 0.20 & $0.00(0.00)$ & $14.43(1.96)$ & $0.57(0.77)$ & $4.80(0.86)$ & $1.36(0.21)$\\
 & 0.25 & $0.01(0.02)$ & $12.47(2.14)$ & $0.47(0.35)$ & $3.35(0.86)$ & $1.18(0.30)$\\
 & 0.30 & $0.00(0.00)$ & $8.46(3.63)$ & $0.93(1.21)$ & $2.71(0.79)$ & $0.99(0.25)$\\
 & 0.35 & $0.00(0.01)$ & $6.44(1.56)$ & $0.43(0.72)$ & $1.83(0.57)$ & $1.01(0.26)$\\
 & 0.40 & $0.00(0.01)$ & $6.22(1.77)$ & $0.18(0.33)$ & $1.67(0.48)$ & $0.86(0.16)$\\
\bottomrule
\end{tabular}
\end{table*}

\begin{table*}[t]
    \centering
    \caption{Mean and standard error of $0\text{-}1\text{-}c$ risk of the noisy-labeled multiclass classification setting (rescaled to $0\text{-}100$).}
    \label{app:noisy-mult-01c}
    \begin{tabular}{lcccccc}
    \toprule
    Dataset & $c$
    & SCE & DEFER & ANGLE & CS-hinge & CS-sigmoid \\
    \midrule
\multirow{7}{*}{Gas-Drift} & 0.10 & $2.44(0.66)$ & $8.29(0.35)$ & $9.66(0.32)$ & $9.16(0.19)$ & $1.88(0.18)$\\
 & 0.15 & $2.80(1.01)$ & $11.14(0.66)$ & $13.73(2.12)$ & $12.88(0.43)$ & $2.10(0.19)$\\
 & 0.20 & $3.13(0.83)$ & $13.54(1.26)$ & $16.95(4.65)$ & $14.49(0.98)$ & $2.08(0.18)$\\
 & 0.25 & $2.57(0.73)$ & $13.75(1.62)$ & $22.02(5.15)$ & $11.38(1.13)$ & $1.94(0.28)$\\
 & 0.30 & $3.34(1.46)$ & $14.43(1.13)$ & $24.32(7.80)$ & $7.33(0.74)$ & $1.67(0.37)$\\
 & 0.35 & $2.44(0.89)$ & $12.43(1.84)$ & $31.49(4.50)$ & $5.33(0.46)$ & $1.59(0.27)$\\
 & 0.40 & $2.48(0.46)$ & $12.92(2.28)$ & $29.77(11.38)$ & $4.42(0.43)$ & $1.52(0.27)$\\
\bottomrule
\end{tabular}
\end{table*}

\begin{table*}[t]
    \centering
    \caption{Mean and standard error of $0\text{-}1$ risk for accepted data of the noisy-labeled multiclass classification setting (rescaled to $0\text{-}100$).}
    \label{app:noisy-mult-err}
    \begin{tabular}{lcccccc}
    \toprule
    Dataset & $c$
    & SCE & DEFER & ANGLE & CS-hinge & CS-sigmoid \\
    \midrule
\multirow{7}{*}{Gas-Drift} & 0.10 & $2.43(0.66)$ & $1.74(0.91)$ & $11.70(29.45)$ & $0.93(0.25)$ & $0.58(0.08)$\\
 & 0.15 & $2.80(1.00)$ & $1.42(0.67)$ & $3.51(3.95)$ & $0.71(0.37)$ & $0.58(0.09)$\\
 & 0.20 & $3.11(0.82)$ & $1.29(0.52)$ & $2.88(2.30)$ & $0.51(0.17)$ & $0.55(0.08)$\\
 & 0.25 & $2.56(0.73)$ & $0.85(0.28)$ & $2.38(1.30)$ & $0.38(0.14)$ & $0.59(0.12)$\\
 & 0.30 & $3.33(1.45)$ & $0.78(0.21)$ & $2.73(1.74)$ & $0.42(0.09)$ & $0.54(0.04)$\\
 & 0.35 & $2.43(0.89)$ & $0.70(0.17)$ & $4.61(2.84)$ & $0.46(0.05)$ & $0.60(0.12)$\\
 & 0.40 & $2.46(0.45)$ & $0.85(0.37)$ & $3.33(1.40)$ & $0.48(0.07)$ & $0.63(0.13)$\\
\bottomrule
\end{tabular}
\end{table*}

\begin{table*}[t]
    \centering
    \caption{Mean and standard error of rejection ratio of the noisy-labeled multiclass classification setting (rescaled to $0\text{-}100$).}
    \label{app:noisy-mult-rej}
    \begin{tabular}{lcccccc}
    \toprule
    Dataset & $c$
    & SCE & DEFER & ANGLE & CS-hinge & CS-sigmoid \\
    \midrule
\multirow{7}{*}{Gas-Drift} & 0.10 & $0.08(0.05)$ & $79.35(3.21)$ & $95.94(3.27)$ & $90.74(1.95)$ & $13.80(1.67)$\\
 & 0.15 & $0.04(0.07)$ & $71.61(4.47)$ & $87.60(21.06)$ & $85.17(2.82)$ & $10.56(1.05)$\\
 & 0.20 & $0.10(0.19)$ & $65.57(6.27)$ & $79.39(32.80)$ & $71.74(4.99)$ & $7.87(1.10)$\\
 & 0.25 & $0.03(0.05)$ & $53.48(6.20)$ & $85.90(26.03)$ & $44.69(4.55)$ & $5.52(1.14)$\\
 & 0.30 & $0.04(0.04)$ & $46.74(3.64)$ & $77.63(32.41)$ & $23.37(2.37)$ & $3.85(1.21)$\\
 & 0.35 & $0.02(0.02)$ & $34.20(5.23)$ & $87.36(18.73)$ & $14.11(1.34)$ & $2.88(0.71)$\\
 & 0.40 & $0.04(0.07)$ & $30.83(5.82)$ & $71.95(31.27)$ & $9.97(1.08)$ & $2.25(0.62)$\\
\bottomrule
\end{tabular}
\end{table*}

\FloatBarrier
\subsection{More discussion on the proposed rejection conditions}
In this section, we provide more discussion on our proposed rejection conditions. 
We provide additional experimental results without using Cond.~\eqref{eq:morethanone-reject}. 
We found that using Cond.~\eqref{eq:morethanone-reject} slightly affects the performance of our cost-sensitive approach. 
This is because there is only a small portion of data that satisfies Cond.~\eqref{eq:morethanone-reject}. 
Note that if we succeed to obtain the optimal classifier $\vecg^*$, this condition is \emph{impossible} to be satisfied.
Recall that in Section~\ref{sec:sub-proposed-multi}, for $\vecg^*$, at most one $g^*_y(\x)$ can be more than zero since it implies $\eta_y > 1-c > 0.5$. 
Nevertheless, Cond.~\eqref{eq:morethanone-reject} may hold in practice due to empirical estimation.

Figures~\ref{fig:app-bin} and \ref{fig:app-mult} illustrate the performance of the proposed approach that uses only Cond.~\eqref{eq:amb-reject} as a rejection condition and other baselines in the binary and multiclass settings, respectively.
Figures~\ref{fig:mnist-rej}, \ref{fig:fmnist-rej}, and \ref{fig:kmnist-rej} illustrate rejected test data based on different conditions for MNIST, Fashion-MNIST, and KMNIST, respectively. 

\begin{figure}
\centering
\includegraphics[width=\linewidth]{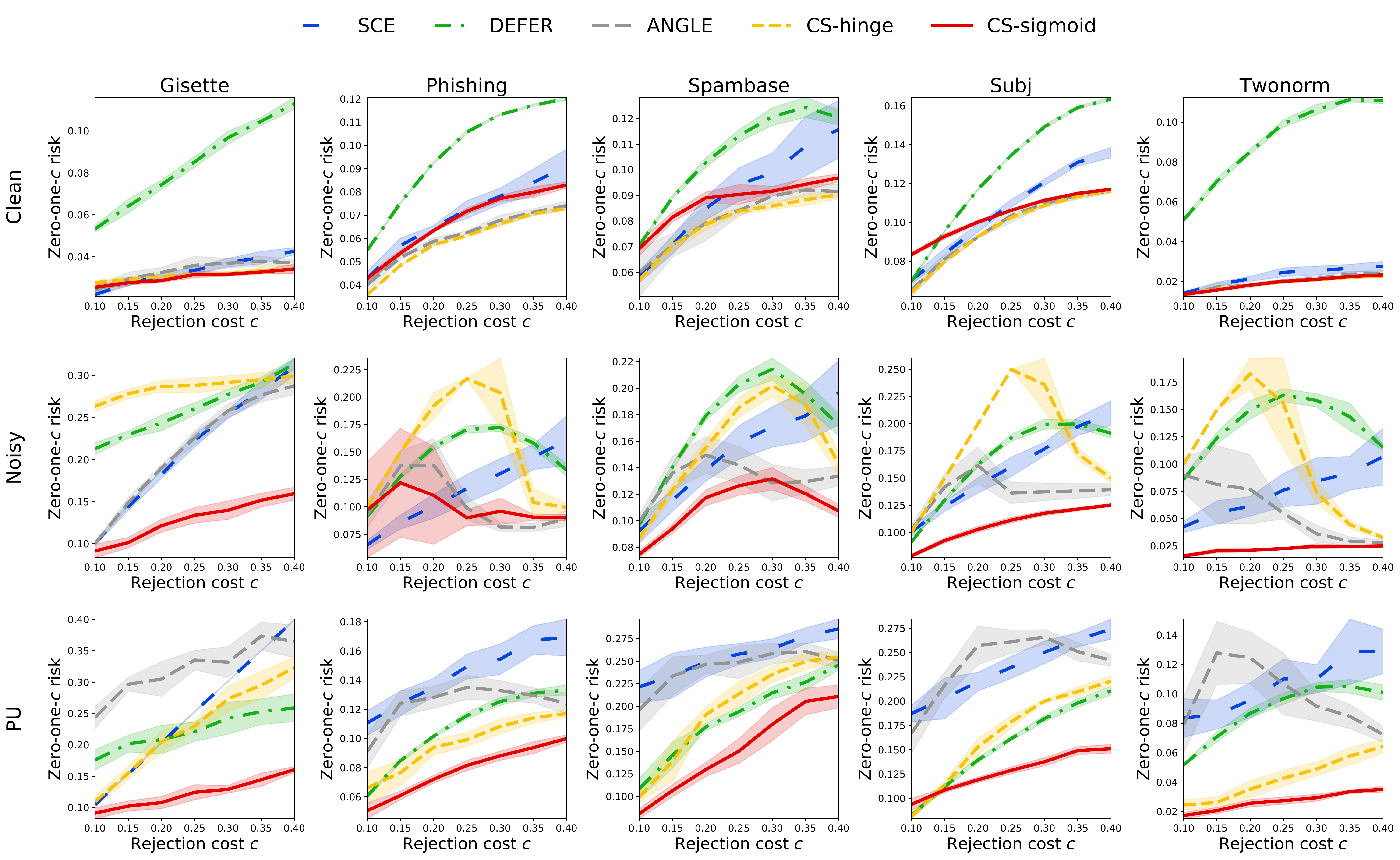}
\caption{Mean and standard error of the test empirical zero-one-$c$ risk over ten trials with varying rejection costs (Binary classification).
Our approach in this figure only used Cond.~\eqref{eq:amb-reject} as a rejection condition.
Each column indicates the performance with respect to one dataset.
(Top) clean-labeled classification with rejection. 
(Middle) noisy-labeled classification with rejection. 
(Bottom) PU-classification with rejection. It can be seen that a similar trend as Figure~\ref{fig:bin-01c} can be observed.}
\label{fig:app-bin}
\end{figure}

\begin{figure}
\centering
\includegraphics[width=\linewidth]{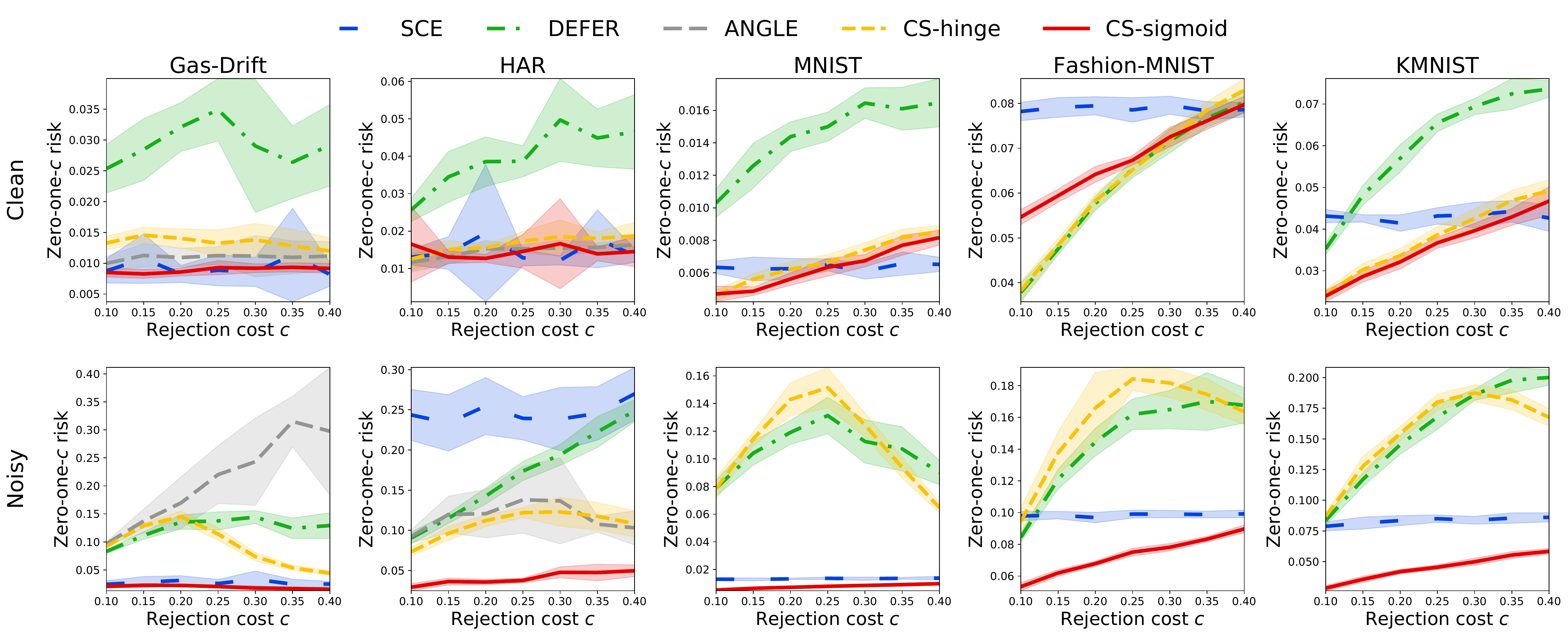}
\caption{
Mean and standard error of the test empirical zero-one-$c$ risk over ten trials with varying rejection cost (Multiclass classification). 
Our approach in this figure only used Cond.~\eqref{eq:amb-reject} as a rejection condition.
Each column indicates the performance with respect to one dataset.
(Top) clean-labeled classification with rejection. 
(Bottom) noisy-labeled classification with rejection.
For MNIST, Fashion-MNIST, and KMNIST, we found that ANGLE failed miserably and has  zero-one-$c$ risk more than $0.5$ and thus it is excluded from the figure for readability. It can be seen that a similar trend as Figure~\ref{fig:mul} can be observed.
}
\label{fig:app-mult}
\end{figure}

\clearpage
\begin{figure}
\centering
\includegraphics[width=.45\linewidth]
{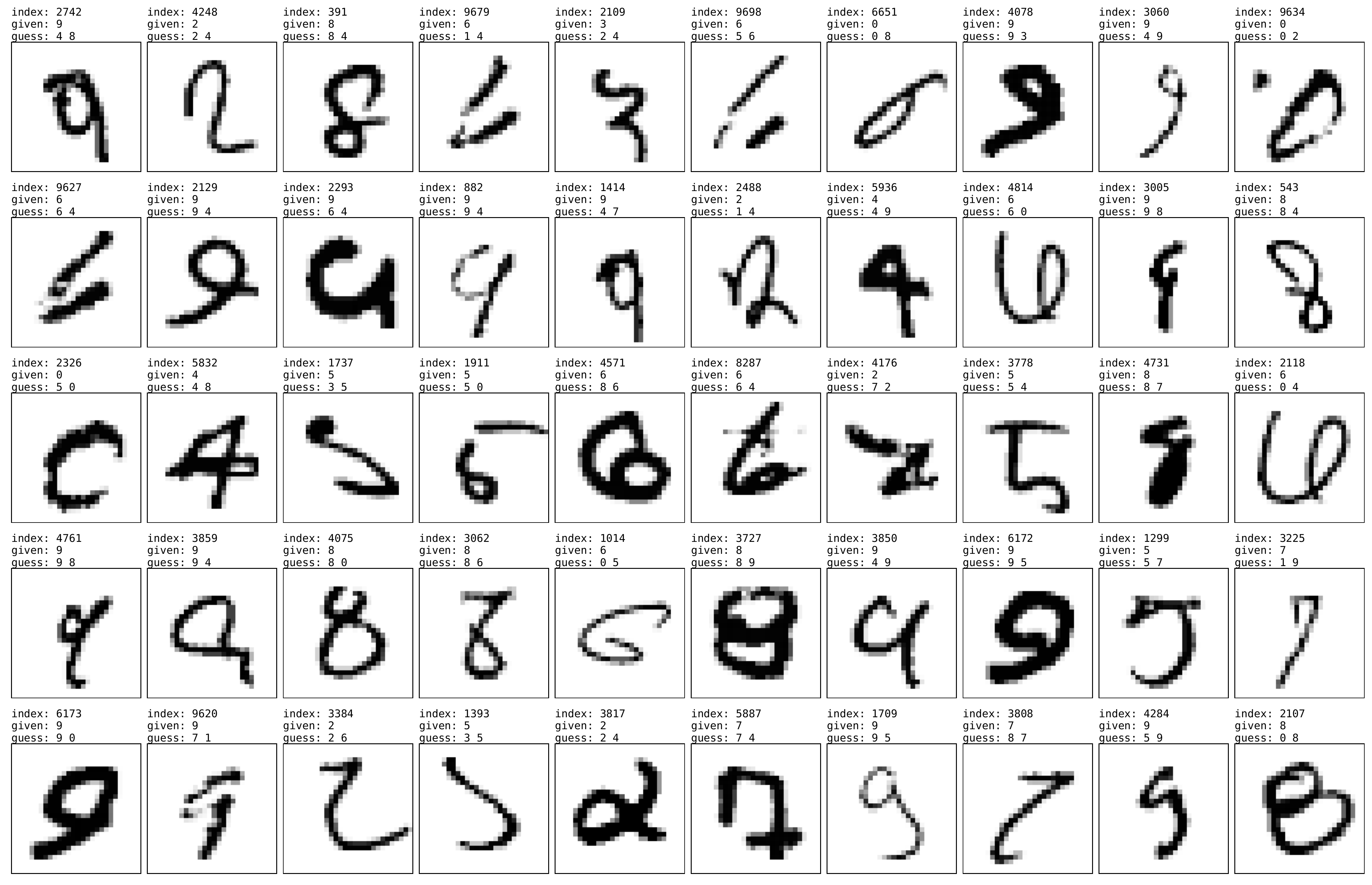}
\quad
\includegraphics[width=.45\linewidth]
{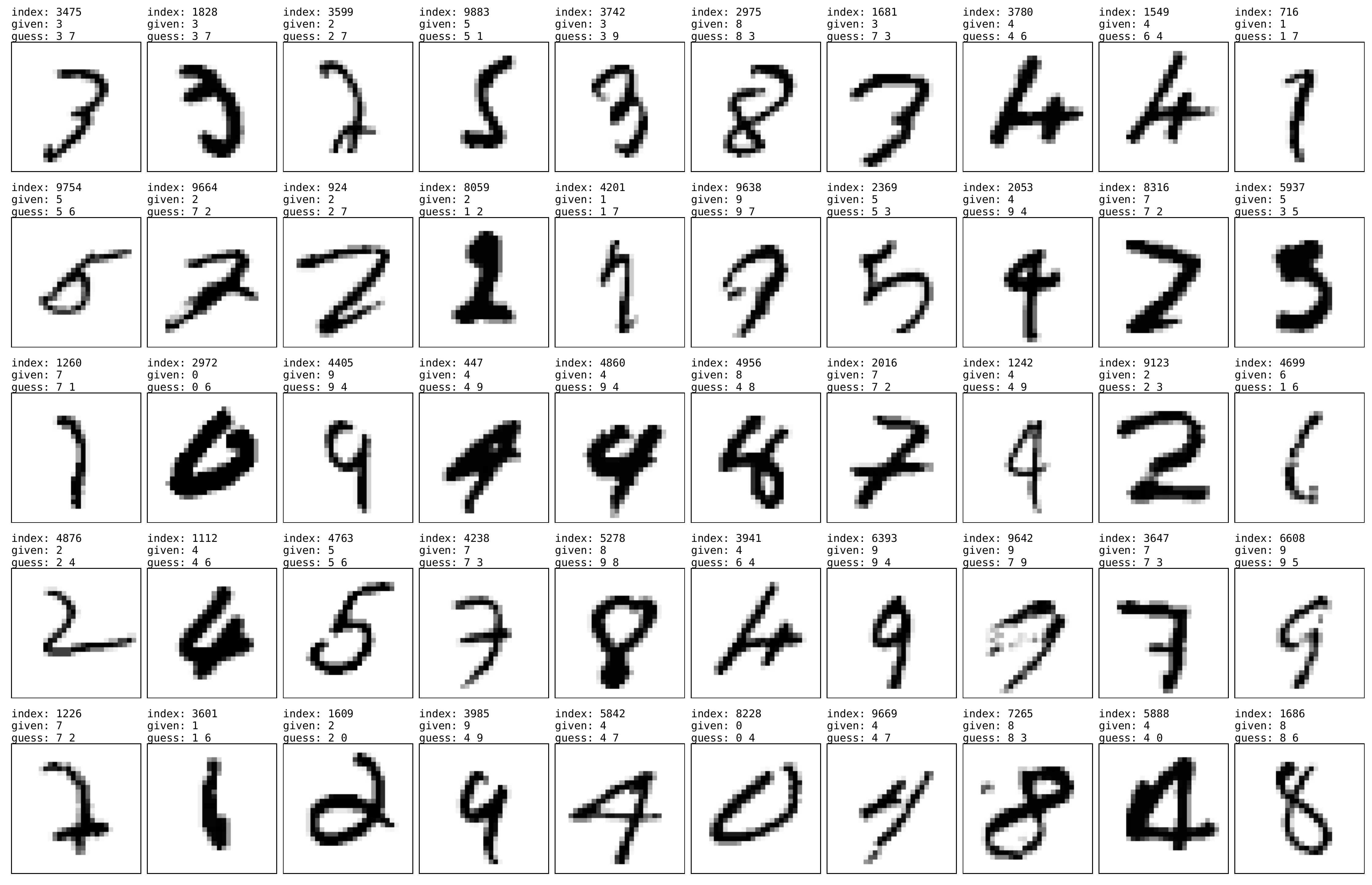}
\caption{%
Examples of rejected test data in MNIST. 
(Left) rejections that were made by Cond.~\eqref{eq:amb-reject} (distance rejection). 
(Right) rejections that were made by Cond.~\eqref{eq:morethanone-reject} (ambiguity rejection). 
The data rejected by Cond.~\eqref{eq:amb-reject} appeared to look more chaotic than the one rejected by Cond.~\eqref{eq:morethanone-reject}. 
On the right figure, several images can be seen to be able to associated to more than one classes, e.g., $1$ vs $7$, $4$ vs $9$, and $3$ vs $5$.
}
\label{fig:mnist-rej}
\end{figure}

\begin{figure}
\centering
\includegraphics[width=.45\linewidth]
{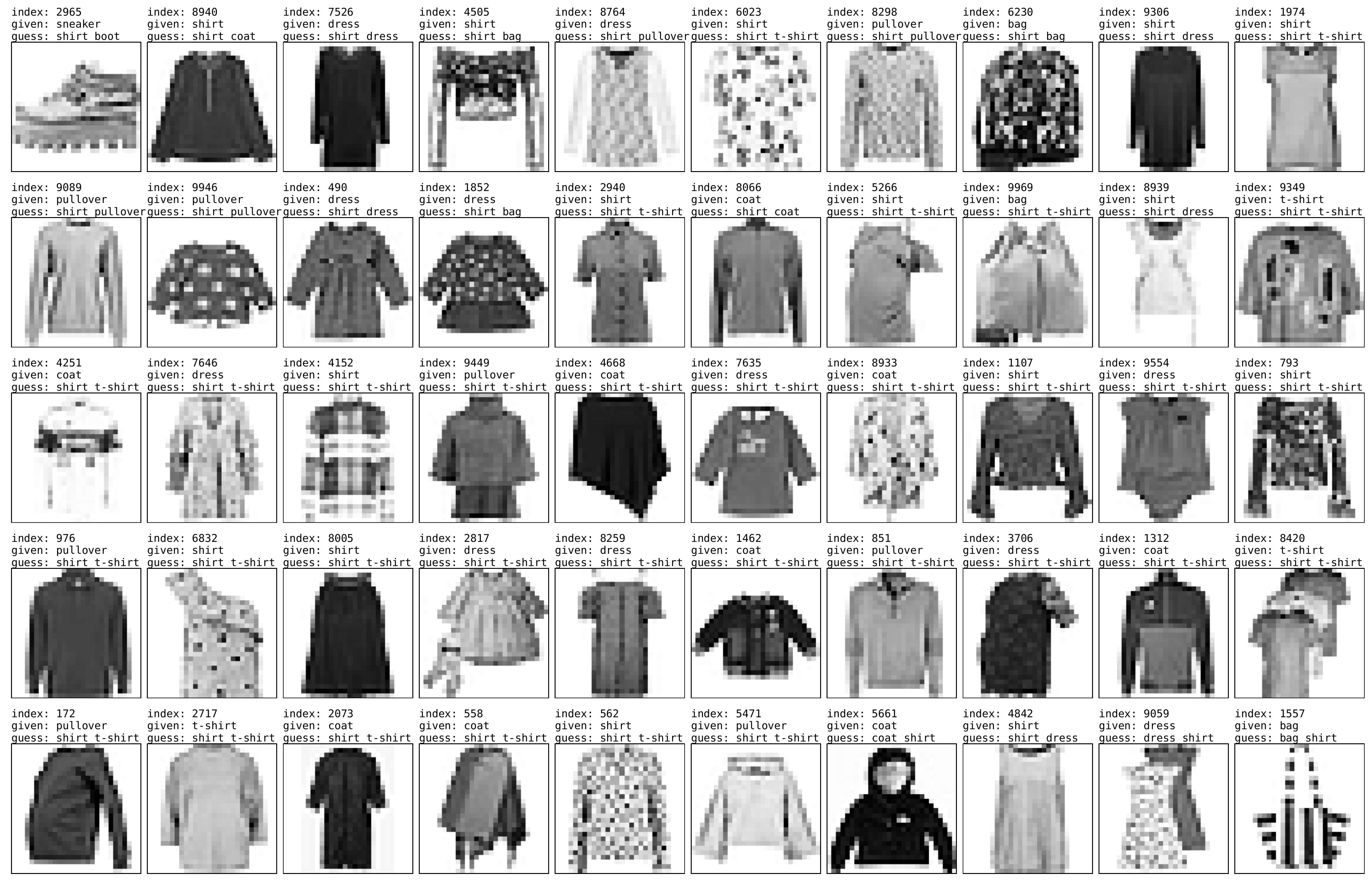}
\quad
\includegraphics[width=.45\linewidth]
{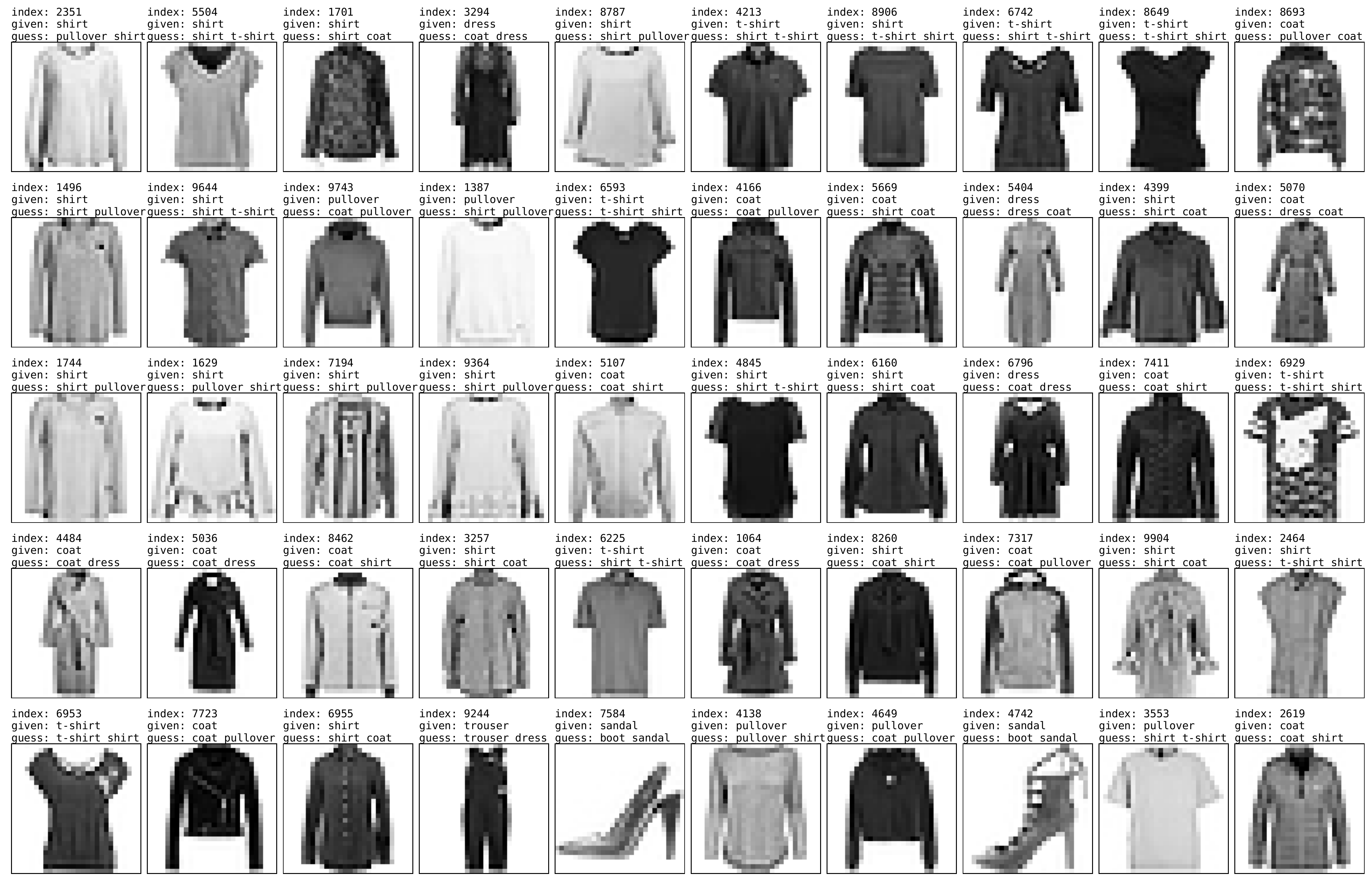}
\caption{%
Examples of rejected test data in Fashion-MNIST. 
(Left) rejections that were made by Cond.~\eqref{eq:amb-reject} (distance rejection). 
(Right) rejections that were made by Cond.~\eqref{eq:morethanone-reject} (ambiguity rejection).
The data rejected by Cond.~\eqref{eq:amb-reject} appeared to have more texture information than the one rejected by Cond.~\eqref{eq:morethanone-reject}.
}
\label{fig:fmnist-rej}
\end{figure}

\begin{figure}
\centering
\includegraphics[width=.45\linewidth]
{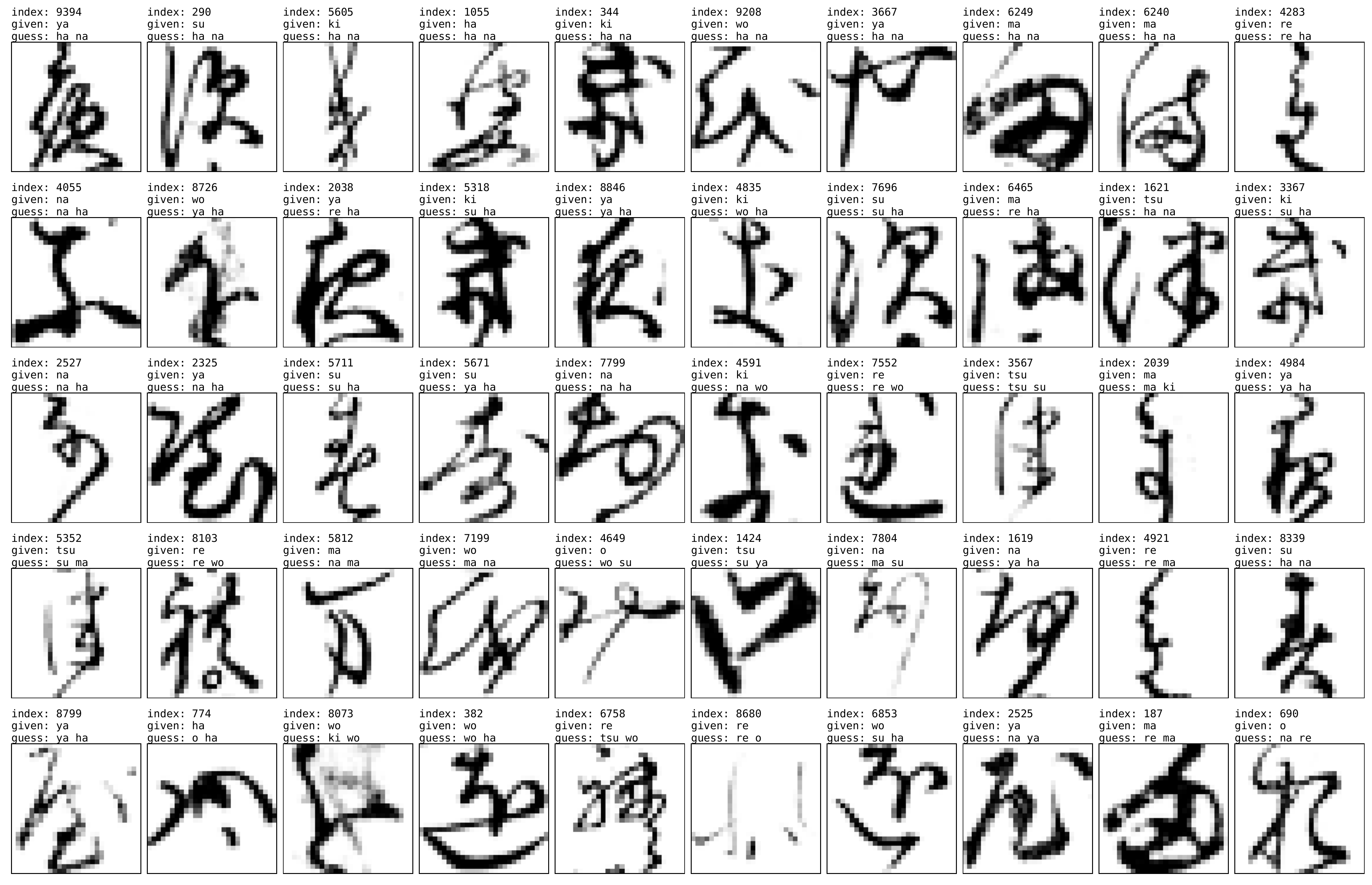}
\quad
\includegraphics[width=.45\linewidth]
{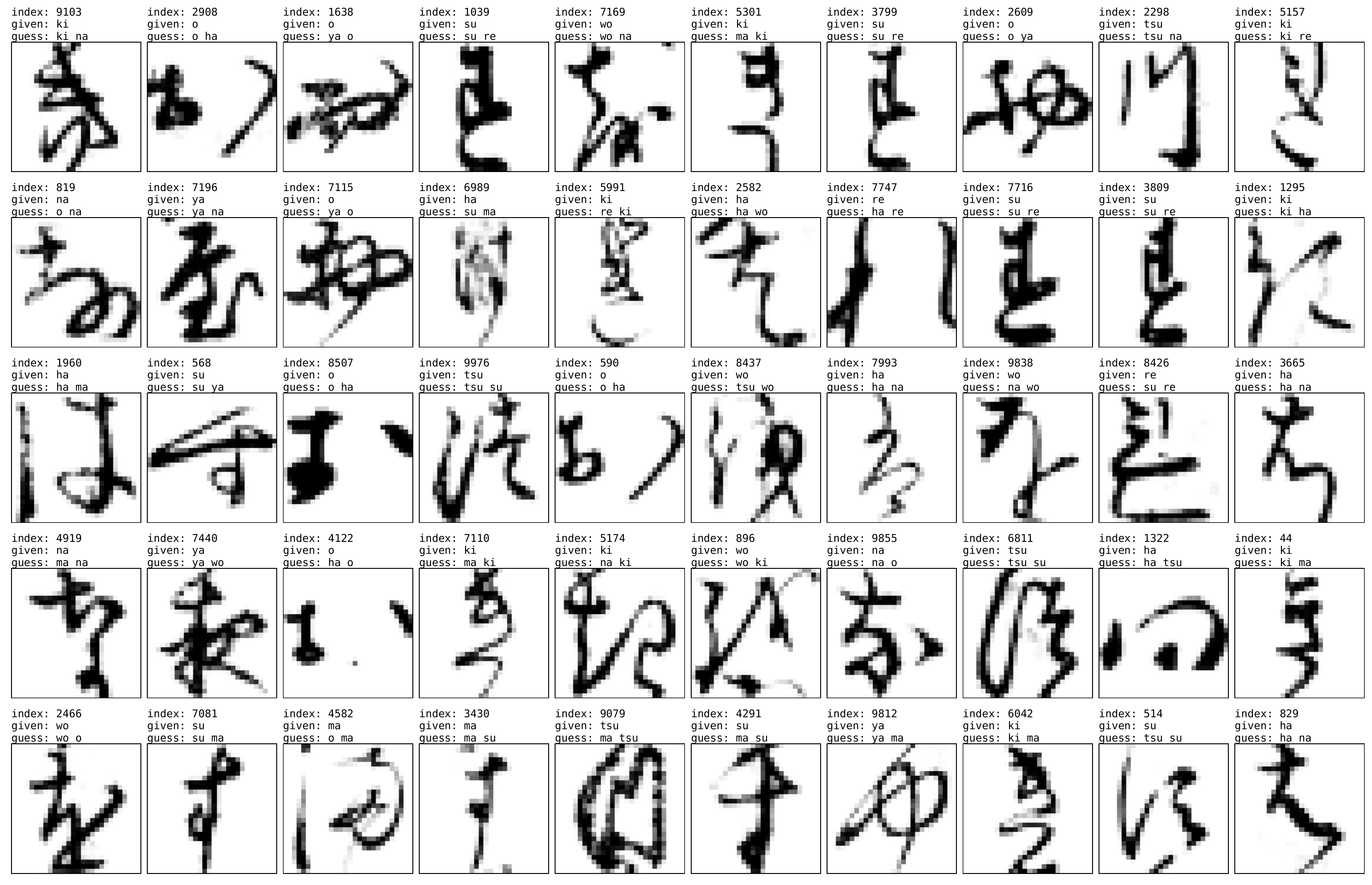}
\caption{%
Examples of rejected test data in KMNIST.
(Left) rejections that were made by Cond.~\eqref{eq:amb-reject} (distance rejection). 
(Right) rejections that were made by Cond.~\eqref{eq:morethanone-reject} (ambiguity rejection). 
The data rejected by Cond.~\eqref{eq:amb-reject} appeared to look more chaotic than the one rejected by Cond.~\eqref{eq:morethanone-reject}.
}
\label{fig:kmnist-rej}
\end{figure}

\end{document}